\documentclass{article}

\usepackage{amsmath,amsfonts,bm}

\def\eqref#1{equation~(\ref{#1})}

\def\1{\bm{1}}

\DeclareMathAlphabet{\mathsfit}{\encodingdefault}{\sfdefault}{m}{sl}
\SetMathAlphabet{\mathsfit}{bold}{\encodingdefault}{\sfdefault}{bx}{n}

\def\gG{{\mathcal{G}}}

\def\gN{{\mathcal{N}}}
\def\gO{{\mathcal{O}}}

\def\gR{{\mathcal{R}}}

\def\gT{{\mathcal{T}}}
\def\gU{{\mathcal{U}}}

\def\gW{{\mathcal{W}}}
\def\gX{{\mathcal{X}}}
\def\gY{{\mathcal{Y}}}

\def\sN{{\mathbb{N}}}

\def\sR{{\mathbb{R}}}

\newcommand{\E}{\mathbb{E}}

\newcommand{\Dep}{\textnormal{Depth}}
\newcommand{\TMD}{\textnormal{TMD}}
\newcommand{\TD}{\textnormal{TD}}
\newcommand{\OT}{\textnormal{OT}}

\DeclareMathAlphabet{\mymathbb}{U}{BOONDOX-ds}{m}{n}

\PassOptionsToPackage{numbers, compress}{natbib}

\usepackage[final]{neurips_2022}

\usepackage[utf8]{inputenc} 
\usepackage[T1]{fontenc}    
\usepackage{hyperref}       
\usepackage{url}            
\usepackage{booktabs}       
\usepackage{amsfonts}       
\usepackage{nicefrac}       
\usepackage{microtype}      
\usepackage{xcolor}         

\usepackage{url}            
\usepackage{booktabs}       
\usepackage{amsfonts}       
\usepackage{nicefrac}       
\usepackage{microtype}      
\usepackage[utf8]{inputenc}

\usepackage{graphicx}
\usepackage{framed}
\usepackage{amssymb}
\usepackage{mathrsfs}
\usepackage{array}
\usepackage{amsthm}
\usepackage{verbatim} 
\usepackage{enumerate}
\usepackage{commath}
\usepackage{amsbsy}
\usepackage{subcaption}
\usepackage{float}

\usepackage{amsmath}
\usepackage{algorithm}
\usepackage{mathtools}
\usepackage{amssymb}
\usepackage[noend]{algpseudocode}
\usepackage{wrapfig}
\usepackage{hyperref}
\usepackage{enumitem}
\usepackage{tabularx}
\usepackage{pifont}
\usepackage{bbold}

\usepackage{tikz}
\newcommand*\circled[1]{\tikz[baseline=(char.base)]{
            \node[shape=circle,draw,inner sep=2pt] (char) {#1};}}

\usepackage{algorithm}
\usepackage[noend]{algpseudocode}

\definecolor{darkblue}{rgb}{0.0,0.0,0.66}  
\definecolor{darkred}{rgb}{100,0.0,0.0} 
\hypersetup{colorlinks=true,linkcolor=darkred,citecolor=darkblue}

\newtheorem{thm}{Theorem}
\newtheorem{definition}[thm]{Definition}

\newtheorem{theorem}[thm]{Theorem}
\newtheorem{lemma}[thm]{Lemma}
\newtheorem{proposition}[thm]{Proposition}

\title{Tree Mover's Distance: Bridging Graph Metrics and Stability of Graph Neural Networks}

\author{%
  Ching-Yao Chuang \\
  MIT CSAIL \\
  \texttt{cychuang@mit.edu} \\
  \And
  Stefanie Jegelka\\
  MIT CSAIL\\
  \small{\texttt{stefje@mit.edu}} \\
}

\begin{document}

\maketitle

\begin{abstract}
Understanding generalization and robustness of machine learning models fundamentally relies on assuming an appropriate metric on the data space. Identifying such a metric is particularly challenging for non-Euclidean data such as graphs. Here, we propose a pseudometric for attributed graphs, the Tree Mover's Distance (TMD), and study its relation to generalization. Via a hierarchical optimal transport problem, TMD reflects the local distribution of node attributes as well as the distribution of local computation trees, which are known to be decisive for the learning behavior of graph neural networks (GNNs). First, we show that TMD captures properties relevant to graph classification: a simple TMD-SVM performs competitively with standard GNNs. Second, we relate TMD to generalization of GNNs under distribution shifts, and show that it correlates well with performance drop under such shifts. The code is available
at \small{\url{https://github.com/chingyaoc/TMD}}.
\end{abstract}

\section{Introduction}

Understanding generalization under distribution shifts -- theoretically and empirically -- relies on an appropriate measure of divergence between data distributions. This, in turn, typically demands a metric on the data space that indicates what kinds of data points are ``close'' to the training data. While such metrics may be readily available for Euclidean spaces, they are more challenging to determine for non-Euclidean data spaces such as graphs with attributes in $\mathbb{R}^p$, which underlie graph learning methods. In this work, we study the question of a suitable metric for message passing graph neural networks (GNNs).

An ``ideal'' metric for studying input perturbations captures the invariances and inductive biases of the model we are examining. For instance, since graph isomorphism is a difficult problem \cite{fortin1996graph}, this metric is expected to be a pseudometric, i.e., will fail to distinguish certain graphs. These ``failures'' should be aligned with the GNNs' invariances. Moreover, several recent works highlight the importance of local structures -- computation trees resulting from unrolling the message passing process -- for the approximation power of GNNs and their inability to distinguish certain pairs of graphs \citep{arvind2020weisfeiler,garg2020generalization,morris19,xu2018powerful}. 
Works on out-of-distribution generalization of GNNs mostly focus on specific types of instances where a trained model may fail miserably \cite{yehudai2021local,xu21extra}, without specifying the behavior for gradual shifts in distributions. \citet{yehudai2021local} show that a sufficiently large, unrestricted GNN may predict arbitrarily on previously unseen computation trees. In practice, we may observe a more gradual change, depending on the magnitude of change of the tree, in terms of both structure and node attributes, and the capacity of the aggregation function.
%
In summary, we desire a metric that reflects the structure of computation trees and the distribution of node attributes \emph{within} trees. Many distances and kernels between graphs have been proposed \cite{borgwardt_book}, several based on local structures \cite{haussler1999convolution}, and some using structure and attributes \cite{shervashidze2011weisfeiler}. Closest to our ideal is the \emph{Wasserstein Weisfeiler-Leman} pseudometric \cite{togninalli2019wasserstein}, which computes an optimal transport distance between node embeddings of two graphs. The embeddings are computed via message passing, which aligns with the computation of GNNs, but loses structural information within trees.

Hence, we propose the \emph{Tree Mover's Distance (TMD)}, a pseudometric on attributed graphs that considers both the tree structure and local distribution of attributes. It achieves this via a hierarchical optimal transport problem that defines distances between trees. First, we observe that the TMD captures properties that capture relationships between graphs and common labels: a simple SVM based on TMD performs competitively with standard GNNs and graph kernels on graph classification benchmarks. Second, we relate TMD to the performance of GNNs under input perturbations. We determine a Lipschitz constant of GNNs with respect to TMD, which enables a bound on their target risk under domain shifts, i.e., distribution shifts between training and test data. This bound uses the metric in two ways: to measure the distribution shift, and to measure perturbation robustness via the Lipschitz constant. Empirically, we observe that the TMD correlates well with the performance of GNNs under distribution shifts, also when compared to other distances. We hence hope that this work inspires further empirical and theoretical work on tightening the understanding of the performance of GNNs under distribution shifts.

In short, this work makes the following contributions:
\vspace{-3mm}
\begin{itemize}[leftmargin=0.5cm]\setlength{\itemsep}{-1pt}
\item We propose a new graph metric via hierarchical optimal transport between computation trees of graphs, which, in an SVM, leads to a competitive graph learning method;
\item We bound the Lipschitz constant of message-passing GNNs with respect to TMD, which allows to quantify stability and generalization; 
\item We develop a generalization bound for GNNs under distribution shifts that correlates well with empirical behavior.
\end{itemize}
\section{Related Works}

\paragraph{Graph Metrics and Kernels}
Measuring distances between graphs has been a long-standing goal in data analysis. However, proper \emph{metrics} that distinguish non-isomorphic graphs in polynomial time are not known. For instance, \citet{vayer2019optimal} propose a graph metric by fusing Wassestein distance \citep{villani2009optimal} and Gromov-Wasserstein distance \citep{memoli2011gromov}. Similar to classic graph metrics \citep{bunke1998graph, sanfeliu1983distance}, the proposed metric requires approximation. Closely related, graph kernels \citep{vishwanathan2010graph,borgwardt_book} have gained attention. Most graph kernels lie in the framework of $\gR$-convolutional \citep{haussler1999convolution}, and measure similarity by comparing substructures. Many $\gR$-convolutional kernels have limited expressive power and sometimes struggle to handle continuously attributed graphs \citep{togninalli2019wasserstein}. In comparison, TMD is as powerful as the WL graph isomorphism test \citep{weisfeiler1968reduction} while accommodating graphs with high dimensional node attributes. Importantly, TMD captures the stability and generalization of message-passing GNNs \citep{kipf2016semi,xu2018powerful}.

\paragraph{Stability and Generalization of Graph Neural Networks} A number of existing works study stability of GNNs to input perturbations. Spectral GNNs are known to be stable to certain perturbations, e.g., size, if the overall structure is preserved \citep{gama2020stability,kenlay2020stability,kenlay2021stability,gama2020stability}. For message passing GNNs,  \citet{yehudai2020size} study perturbations of graph size, and demonstrate the importance of local computation trees. \citet{xu21extra} study how the out-of-distribution behavior of aggregation functions may affect the GNN's prediction. These studies motivate to include both computation trees and inputs to aggregation functions in the TMD. Finally, a number of works study within-distribution generalization \cite{du2019graph,garg2020generalization,liao2020pac,scarselli2018vapnik}.

\section{Background on Optimal Transport}
\label{sec_ot_intro}

We begin with a brief introduction to Optimal Transport (OT) and earth mover's distance. The earth mover's distance, also known as Wasserstein distance, is a distance function defined via the transportation cost between two distributions. Let $X = \{x_i\}_{i=1}^m$ and $Y = \{y_i\}_{j=1}^m$ be two multisets of $m$ elements each. Let $C \in \sR^{m \times m}$ be the transportation cost for each pair: $C_{ij} = d(x_i, y_j)$ where $d$ is the distance between $x_i$ and $y_j$. The
earth mover's distance solves the following OT problem:
\begin{align}
    \OT_d^\ast(X, Y) := \min_{\gamma \in \Gamma(X, Y)} \langle C, \gamma \rangle / m, \quad \Gamma(X, Y) = \{ \gamma \in \sR_+^{m \times m} \;|\; \gamma \mathbb{1}_{m} = \gamma^\top \mathbb{1}_{m} = \mathbb{1}_{m} \},
\end{align}
where $\Gamma$ is the set of \emph{transportation plans} that satisfies the flow constrain $\gamma \mathbb{1}_{m} = \gamma^\top \mathbb{1}_{m} = \mathbb{1}_{m}$. In this work, we adopt the \emph{unnormalized} version of the earth's mover distance:
\begin{align}
    \label{eq_un_emd}
    \OT_d(X, Y) := \min_{\gamma \in \Gamma(X, Y)} \langle C, \gamma \rangle = m \cdot  \OT_d^\ast(X, Y).
\end{align}
Comparing to classic OT, unnormalized OT preserves the size information of multisets $X$ and $Y$.

\section{Tree Mover's Distance: Optimal Transport on Graphs}

Next, we introduce \emph{tree mover's distance}, a new distance for graphs. Let $G = (V, E)$ denote a graph with sets of nodes $V$ and edges $E$. The graph may have node features $x_v \in \sR^p$ for $v \in V$. If not, we simply set the node feature to a scalar $x_v = 1$ for all $v \in V$. 

Local structures of graphs are characterized by \emph{computation trees} \citep{jegelka2022theory}. In particular, computation trees are constructed by connecting adjacent nodes recursively.

\begin{definition}[Computation Trees]
Given a graph $G = (V, E)$, let $T_v^1 = v$, and let $T_v^L$ be the depth-$L$ computation tree of node $v$ constructed by connecting the neighbors of the leaf nodes of $T_v^{L-1}$ to the tree. The multiset of depth-$L$ computation trees defined by $G$ is denoted by $\gT_G^L := \{T_v^L\}_{v \in V}$.
\end{definition}
Figure \ref{fig_tmd} illustrates an example of constructing computation trees. Computation trees, also referred to as subtree patterns, have been central in graph analysis \citep{pearson1905problem, weisfeiler1968reduction} and graph kernels \citep{ramon2003expressivity, shervashidze2011weisfeiler}. Intuitively, the computation tree of a node encodes local structure by appending the neighbors to the tree in each level. If depth-$L$ computation trees are the same for two nodes, they share similar neighborhoods up to $L$ steps away. Therefore, an intuitive way to compare two graphs is by measuring the difference of their nodes' computation trees \citep{shervashidze2011weisfeiler, weisfeiler1968reduction}. In this work, we adopt optimal transport, a natural way to compute the distance between two sets of objects with, importantly, an underlying geometry. We will begin by defining the transportation cost between two computation trees. The cost then gives rise to the tree mover's distance, an extension of earth mover's distance to multisets of trees.

\subsection{Distance between Trees via Hierarchical OT}

Let $T = (V, E, r)$ denote a rooted tree. We further let $\gT_v$ be the multiset of computation trees of the node $v$ which consists of trees that root at the descendants of $v$.
Determining whether two trees are similar requires iteratively examining whether the subtrees in each level are similar. For instance, two trees $T_a$ and $T_b$ with roots $r_a$ and $r_b$ are the same if $x_{r_a} = x_{r_b}$ and $\gT_{r_a} = \gT_{r_b}$. This motivates us to define the distance between two trees by recursively computing the optimal transportation cost between their subtrees. Nevertheless, the number of subtrees could be different for $r_a$ and $r_b$, i.e., $|\gT_{r_a}| \neq |\gT_{r_b}|$ (see Figure \ref{fig_tree_ot} left). To compute the OT between sets with different sizes, unbalanced OT \citep{chizat2015unbalanced, sato2020fast} or partial OT \citep{chapel2020partial} are usually adopted. Inspired by \citep{chapel2020partial}, we augment the smaller set with \emph{blank trees}.

\begin{figure*}[tbp]
\begin{center}   
\includegraphics[width=0.75\linewidth]{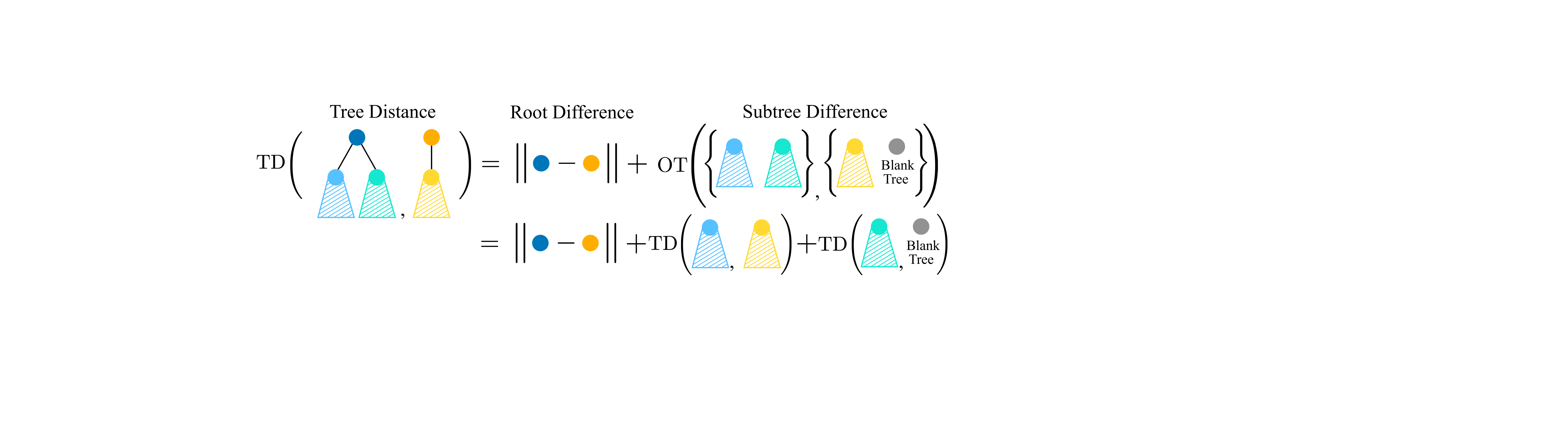}
\end{center}
\vspace{-1mm}
\caption{\textbf{Tree Distance via Hierarchical OT.} The weight $w(\cdot)$ is set to 1 to simplify visualization. The distance between two trees can be decomposed into (a) the distance between roots and (b) the OT cost between subtrees, where the cost function of OT is again the distance between two trees. This formulates a \emph{hierarchical} OT problem, where solving the OT between trees requires solving the OT between subtrees.} 
\vspace{-3mm}
\label{fig_tree_ot}
\end{figure*}

\begin{definition}[Blank Tree]
A \emph{blank tree} $T_\mymathbb{0}$ is a tree (graph) that contains a single node and no edge, where the node feature is the zero vector $\mymathbb{0}_p \in \sR^p$, and $T_\mymathbb{0}^n$ denotes a multiset of $n$ blank trees. 
\end{definition}

\begin{definition}[Blank Tree Augmentation]
Given two multisets of trees $\gT_u, \gT_v$, define $\rho$ to be a function that augments a pair of trees with blank trees as follows:
\begin{align*}
    \rho: (\gT_v, \gT_u) \mapsto \left(\gT_v \bigcup T_\mymathbb{0}^{\max(|\gT_u| - |\gT_v|, 0)}, \gT_u \bigcup T_\mymathbb{0}^{\max(|\gT_v| - |\gT_u|, 0)} \right).
\end{align*}
\end{definition}
If $|\gT_v| < |\gT_u|$, $\rho$ augments $\gT_v$ by $|\gT_u| - |\gT_v|$ blank trees to make the two multisets contain the same number of trees, and hence allows to define a transportation cost between two multisets of trees with different sizes. In particular, the transportation costs of additional trees are simply the distance to the blank trees. The distance between a tree and a blank tree can be interpreted as the \emph{norm of the tree}. In our case, the blank tree can be viewed as the origin, as it is a tree with the simplest structure and zero feature vector. Equipped with $\rho$, we define the distance between two rooted trees as follows.

\begin{definition}[Tree Distance]
\label{def_ot_tree}
The distance between two trees $T_a, T_b$ is defined recursively as
\begin{align*}
 \TD_{w}(T_a,T_b) &:= \begin{cases}
       \|x_{r_{a}} - x_{r_{b}} \| + w(L) \cdot \OT_{\TD_{w}}(\rho(\gT_{r_a}, \gT_{r_b})) & \text{if $L > 1$}\\
       \|x_{r_{a}} - x_{r_{b}} \| & \text{otherwise},
    \end{cases}     
\end{align*}
where $L = \max(\textnormal{Depth}(T_a), \textnormal{Depth}(T_b))$ and $w: \sN \rightarrow \sR^+$ is a depth-dependent weighting function.
\end{definition}

Here, $\OT_{\TD_{w}}$ is the OT distance defined in (\ref{eq_un_emd}) with $\TD_{w}$ as the metric. Figure \ref{fig_tree_ot} gives an illustration of computing tree distances. The tree distance $\TD_w(T_a, T_b)$ aims to optimally align two trees $T_a$ and $T_b$ by recursively comparing their roots and subtrees. Calculating $\TD_w(T_a,T_b)$ requires calculating the OT between augmented subtrees $\rho(\gT_{r_a}, \gT_{r_b})$, where the cost function of OT is $\TD_w$ again. This formulates a \emph{hierarchical} optimal transport problem: the distance of two trees is defined via the distances of subtrees, where the importance of each level is determined by the weight $w(\cdot)$. Increasing $w(\cdot)$ upweights the effect of nodes in the subtrees. While the weights may be arbitrary, we found that for many applications, using a single weight for all depths yields good empirical performance, as section \ref{sec_app} shows. The role of weights will be more significant when we use TMD to bound the stability of GNNs.

\subsection{From Tree Distance to Graph Distance}

Next, we extend the distance between trees to a distance between graphs. By leveraging the tree distance $\TD_w(\cdot,\cdot)$, 
we introduce the \emph{tree mover's distance (TMD)}, a distance for graphs, by calculating the optimal transportation cost between the graphs' computation trees.

\begin{definition}[\textbf{Tree Mover's Distance}]
Given two graphs $G_a, G_b$ and $w, L \geq 0$, the tree mover's distance between $G_a$ and $G_b$ is defined as
\begin{align*}
\TMD_{w}^L(G_a, G_b) = \OT_{\TD_w}(\rho(\gT_{G_a}^L, \gT_{G_b}^L)),
\end{align*}
where $\gT_{G_a}^L$ and $\gT_{G_b}^L$ are multisets of the depth-$L$ computation trees of graphs $G_a$ and $G_b$, respectively.
\end{definition}
Figure \ref{fig_tmd} illustrates the computation of TMD. Intuitively, TMD is the minimum cost required to transport node-wise computation trees from one graph to another. The blank tree augmentation $\rho$ is again adopted to handle graphs with different numbers of nodes. The next theorem shows that TMD is a pseudometric on attributed graphs. 

\begin{theorem}[Pseudometric]
The tree mover's distance $\TMD_{w}^L$ is a pseudometric for finite $L > 0$. 
\end{theorem}
In particular, the tree mover's distance satisfies (1) $\TMD_{w}^L(G_a,G_a) = 0$, (2) $\TMD_{w}^L(G_a,G_b) = \TMD_{w}^L(G_b,G_a)$, and (3) $\TMD_{w}^L(G_a, G_b) \leq \TMD_{w}^L(G_a, G_c) + \TMD_{w}^L(G_c,G_b)$ for any graphs $G_a, G_b, G_c$. However, in some cases, the distance $\TMD_{w}^L(G_a, G_b)$ can be zero even if $G_a \neq G_b$. This is reasonable, as computing graph isomorphism is not known to be solvable in polynomial time. Nevertheless, TMD can provably distinguish graphs that are identifiable by the (1-dimensional) Weisfeiler-Leman graph isomorphism test \citep{weisfeiler1968reduction}.

\begin{theorem} [Discriminative Power of TMD]
\label{thm_tmd_power}
If two graphs $G_a$, $G_b$ are determined to be non-isomorphic in WL iteration $L$ and $w(l) > 0 $ for all $0<l \leq L+1$, then $\TMD_{w}^{L+1}(G_a,G_b) > 0$.
\end{theorem}
The unnormalized OT and blank tree augmentation are essential to prove Theorem \ref{thm_tmd_power}. The tree mover's distance can be exactly computed by solving optimal transport. In addition, TMD remains highly expressive on graphs with high dimensional continuous attributes, where most $\gR$-convolutional graph kernels struggle \citep{togninalli2019wasserstein}. The discriminative power of TMD can be further strengthened by augmenting node attributes e.g.\ with positional encodings \citep{dwivedi2021graph, lim2022sign}.

The OT cost between node representations in different graphs is reminiscent of the recently proposed Wasserstein WL (WWL) \citep{togninalli2019wasserstein}. WWL uses the distance of node embeddings as a ground metric, where the node embeddings are computed via $L$ iterations of message passing with average-aggregation; in contrast, TMD computes a distance that aligns the trees, retaining more structural information. Even though an aggregation with nonlinearities can retain tree isomorphism information \citep{xu2018powerful}, similar to the hashing applied in the WL test (WWL uses only linear aggregations), the hierarchical OT is a more direct graded distance measure of trees. \citet{vayer2019optimal} define a metric that uses a Gromov-Wasserstein distance \citep{memoli2011gromov} between nodes, but need to approximate the GW computation.

\begin{figure*}[tbp]
\begin{center}   
\includegraphics[width=\linewidth]{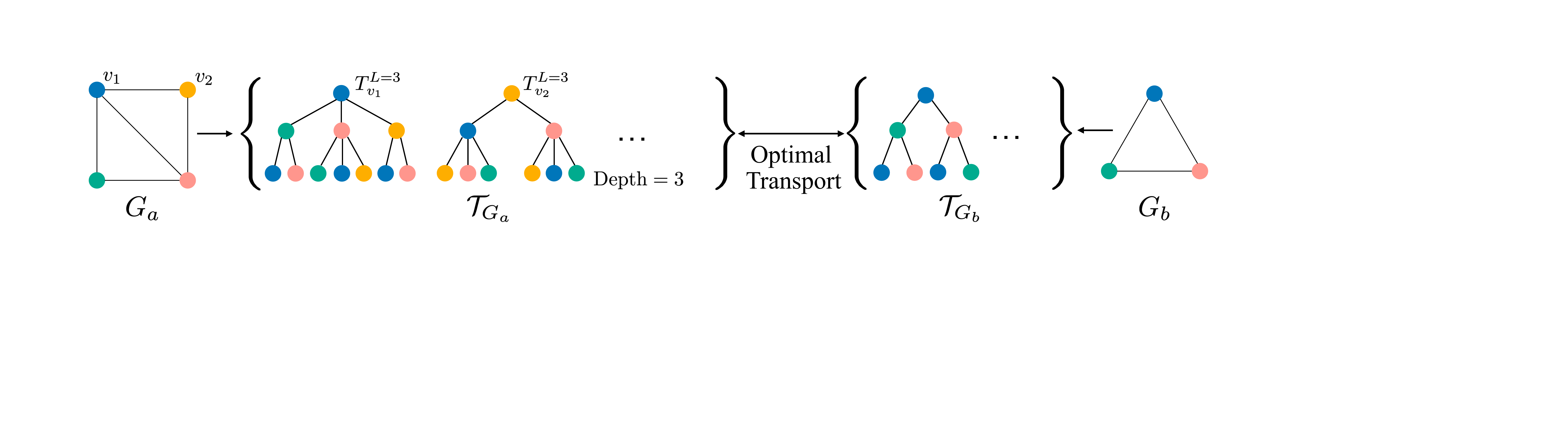}
\end{center}
\caption{\textbf{Illustration of Computation Trees and Tree Mover's Distance.} The computation trees of nodes are constructed by iteratively connecting the neighbors to the trees, and each graph will define a multiset of node-wise computation trees. Tree mover's distance is then defined as the optimal transport cost between the computation trees of two graphs.} 
\label{fig_tmd}
\end{figure*}

\paragraph{Numerical Computation with Dynamic Programming}
One can numerically compute TMD with dynamic programming. Starting with pair-wise distances ($\TMD_w^{L=1}$) between node features across the two graphs, we then iteratively compute $\TMD_w^{L=k}$ between depth-$k$ computation trees from $k=2$ to $L$ according to Definition \ref{def_ot_tree}. Let $D$ be the maximum degree of a node in the two graphs and $\tau(m)$ be the complexity of computing OT between sets of cardinality $m$. In each level, we have to perform OT of sets contain at most $D$ elements for $N$ nodes. Including the last OT between nodes of graph, the overall time complexity of computing $\TMD_w^L$ is $\gO(\tau(N) + L N \tau(D))$. The time complexity for exact OT by solving linear programming is $\tau(m) = \gO(m^3 \log(m))$ \citep{flamary2021pot}. One can use faster approximation of OT, e.g., near linear time complexity \citep{altschuler2017near}, but we use exact OT throughout all the experiments, implemented with the POT library \citep{flamary2021pot}.

\subsection{Experiments}
\label{sec_app}

\begin{table*}[!t]
\small
\begin{center}{%
\begin{tabularx}{0.97\textwidth}{l| *{7}{c}  }
\toprule
 & MUTAG & PTC & PROTEINS & NCI1 & NCI109  & BZR & COX2 \\
\midrule
TMD L=1 & 89.4$\pm$5.5 & 65.3$\pm$5.8 & 73.9$\pm$2.8 & 68.3$\pm$2.0 & 69.5$\pm$1.6  & 83.8$\pm$7.2 & 77.8$\pm$5.0
\\
TMD L=2 & 90.0$\pm$5.7 & 67.4$\pm$7.7 & 74.8$\pm$2.8 & 80.8$\pm$1.8 & 78.9$\pm$2.3 & 84.5$\pm$6.9  & \textbf{79.1$\pm$5.2}
\\
TMD L=3 & 91.1$\pm$5.4 & \textbf{68.5$\pm$6.1} & 74.6$\pm$2.6 & 83.3$\pm$1.1 & 82.3$\pm$2.5  & \textbf{85.5$\pm$6.2} & 78.5$\pm$5.9
\\
TMD L=4 & \textbf{92.2$\pm$6.0} & 66.5$\pm$7.1 & 75.2$\pm$2.3 & 84.8$\pm$1.2 & \textbf{82.8$\pm$2.1} & 84.5$\pm$6.4 & 76.1$\pm$6.1
\\
\midrule
\midrule 
WWL \citep{togninalli2019wasserstein} & 87.3$\pm$1.5 & 66.3$\pm$1.2 & 74.3$\pm$0.6 & 86.1$\pm$0.3 & - & 84.4$\pm$2.0 & 78.3$\pm$0.5
\\
FGW \citep{vayer2019optimal} & 88.4$\pm$5.6 & 65.3$\pm$7.9 & 74.5$\pm$2.7 & \textbf{86.4$\pm$1.6} & - & 85.1$\pm$4.2 & 77.2$\pm$4.9
\\
WL \citep{shervashidze2011weisfeiler} & 90.4$\pm$5.7 & 59.9$\pm$4.3 & 75.0$\pm$3.1 & 86.0$\pm$1.8 & 82.46$\pm$0.2 & N/A & N/A
\\
R\&G \citep{ramon2003expressivity} & 85.7$\pm$0.4 & 58.5$\pm$0.9 & 70.7$\pm$0.4 & 61.9$\pm$0.3 & 61.7$\pm$0.2 & N/A & N/A
\\

\midrule
\midrule
GIN \citep{xu2018powerful} & 89.4$\pm$5.6 & 64.6$\pm$7.0 & \textbf{76.2$\pm$2.8} & 82.7$\pm$1.7 & 82.2$\pm$1.6  & 83.5$\pm$6.0 & 79.0$\pm$5.3
\\
GCN \citep{kipf2016semi} & 85.6$\pm$5.8 & 64.2$\pm$4.3 & 76.0$\pm$3.2 & 80.2$\pm$2.0 & - & 84.6$\pm$5.9 & 77.1$\pm$4.7
\\
\bottomrule
 \end{tabularx}}
\end{center}
\caption{\textbf{Classification on TU Dataset.} TMD outperforms or matches the state-of-the-art graph kernels or GNNs. Note that WL \citep{shervashidze2011weisfeiler} and R\&G \citep{ramon2003expressivity} are not applicable to continuously attributed graphs such as BZR and COX2.
}
\label{table_classification}
\end{table*}

We verify whether the tree mover's distance aligns with the labels of graphs in graph classification tasks: the TUDatasets \citep{morris2020tudataset}, which contain graphs with discrete node attributes (MUTAG, PTC-MR, PROTEINS, NCI1, NCI109) and graphs with continuous node attributes (BZR, COX2). Specifically, we run a support vector classifier (C$=$1) with indefinite kernel $e^{-\gamma \times \TMD(\cdot, \cdot)}$, which can be viewed as a noisy observation of the true positive semidefinite kernel \citep{luss2007support}. The $\gamma$ is selected via cross-validation from $\{0.01, 0.05, 0.1\}$ and the weights $w(\cdot)$ are set to $0.5$ for all depths. For comparison, we use graph kernels based on graph subtrees: Ramon \& G\"{a}rtner kernel \citep{ramon2003expressivity}, WL subtree kernel \citep{shervashidze2011weisfeiler}; two widely-adopted GNNs: graph isomorphism network (GIN)  \citep{xu2018powerful}, graph convolutional networks (GCN) \citep{kipf2016semi}; and the recently proposed graph metrics FGW \citep{vayer2019optimal} and WWL \citep{togninalli2019wasserstein}. Table \ref{table_classification} reports the mean and standard deviation over 10 independent trials with 90\%/10\% train-test split. The performances of the baselines are taken
from the original papers. TMD outperforms or matches the performances of state-of-the-art GNNs, graph kernels, and metrics, implying that it captures meaningful structural properties of graphs. Appendix \ref{appendix_exp} shows further graph clustering and t-SNE visualization \citep{van2008visualizing} results.

\begin{wraptable}{r}{0.4\textwidth}
\small
\begin{center}{%
\begin{tabularx}{0.4\textwidth}{l| *{3}{c}  }
\toprule
 & WWL & TMD & TMD Parallel \\
\midrule
DD & 7.92 & 32.07 & 24.44
\\
NCI1 & 0.11 & 0.34 & 0.81
\\
\bottomrule
 \end{tabularx}}
\end{center}
\vspace{-2mm}
\caption{\textbf{Runtime Comparison.} The average runtime (sec/pair) of TMD is much faster than the (worst case) theoretical Big-O analysis.
}
\label{table_runtime}
\end{wraptable}

\paragraph{Computation Complexity} To examine the computation complexity, we compare the runtime between Weisfeiler-Lehman (WWL) kernels, TMD, and a parallel version of TMD on DD and NCI1 datasets \citep{morris2020tudataset}, where DD contains large graphs (avg. \#node: 284.32, avg. \#edge: 715.66) and NCI1 contains small graphs (avg. \#node: 29.87, avg. \#rdge: 32.30). Here, we additionally consider a parallel version of TMD, where the tree OTs in each level are executed simultaneously with 3 processes. The average runtime over 200 pairs is shown in Table \ref{table_runtime}. Note that the time complexity of WWL is $\tau(m) = \gO(m^3 \log(m))$. The runtime of parallelized TMD is roughly three times larger than WWL on DD, which is much faster than the (worst case) theoretical Big-O analysis. In datasets contain small graphs such as NCI1, TMD without parallelization works well.

\vspace{-2mm}
\section{TMD and Stability of Graph Neural Networks}
\label{sec_stability}
\vspace{-2mm}
Next, we relate TMD to the perturbation stability of message passing GNNs. In particular, we use the Lipschitz constant, which relies on an underlying metric -- here, a metric over graphs. We observe that TMD is a meaningful pseudometric in this case.
The resulting Lipschitz allows to analyze the stability of GNNs under perturbations and generalization bounds under distribution shifts.

\vspace{-3mm}
\subsection{Lipschitz Constant of Message Passing Graph Neural Networks}\label{sec:lipschitz}
\vspace{-3mm}

For simplicity, we consider graph binary classification with the Graph Isomorphim Network (GIN) \citep{xu2018powerful}, one of the most widely applied and powerful GNNs. In particular, we consider the following message passing rules of a $L$-layer GIN: 
\begin{align*}
    \begin{Large}  \substack{\textnormal{Message} \\ \textnormal{Passing} } \end{Large} \;\; z_v^{(l)} = \phi^{(l)} \left(z_v^{(l-1)} + \epsilon \sum_{u \in \gN(v)} z_u^{(l-1)} \right), \;\;\; \begin{Large}  \substack{\textnormal{Graph}\\ \textnormal{Readout}} \end{Large} \;\; h(G) = \phi^{(L+1)} \left( \sum_{u \in V} z_u^{(L)} \right)
\end{align*}
where $\phi^{(l)}: \sR^d \rightarrow \sR^d$ and $\phi^{(L+1)}: \sR^d \rightarrow \sR$ are learnable functions with Lipschitz constant $K_\phi^{(l)}$ and the initial state is set to the node feature $z_v^{(0)} = x_v$. The $\epsilon > 0$ is a weighting term between the center node and the neighbors. Note that the original formulation of GIN \citep{xu2018powerful} weights the center node with layer dependent $\epsilon$ instead of neighbors. We adopt the form above with the purpose to simplify the notation of TMD. One can easily derive an equivalent form by changing the weight function $w(\cdot)$ of TMD to recover the original formulation of GIN. For simplicity, we set $\epsilon = 1$ in all experiments. This does not affect the empirically performance of GIN as the original paper shows \citep{xu2018powerful}. A graph level binary classifier is constructed based on the logits $h(G)$. 

The next theorem bounds the Lipschitz constant of GIN with respect to TMD. Although TMD is a pseudometric, it satisfies that $\TMD_{w}^L(G_a, G_b) > 0$ if $G_a, G_b$ are distinguished by $L$ iterations of the WL test. Since $\mathrm{GIN(G_a)} = \mathrm{GIN(G_b)}$ for all graphs where WL fails, it holds that, if  $\mathrm{GIN(G_a)} \neq \mathrm{GIN(G_b)}$ then $\TMD_{w}^L(G_a, G_b) > 0$. That is, for all triples of graphs that GIN distinguishes pairwise, TMD is a metric.

\begin{theorem}[\textbf{Lipschitz Constant of GIN}]
\label{thm_wass_lip}
Given an $L$-layer graph neural network $h: \gX \rightarrow \sR$ and two graph $G_a, G_b \in \gG$, we have
\begin{align*}
    &\left \| h(G_a) - h(G_b) \right \|  \leq  \prod_{l=1}^{L+1} K_\phi^{(l)}  \cdot  \TMD_{w}^{L+1}(G_a, G_b),
\end{align*}
where $w(l)$=$\epsilon \cdot P_{L+1}^{l-1} / P_{L+1}^{l}$ for all $l \leq L$ and $P_{L}^l$ is the $l$-th number at level $L$ of Pascal's triangle.
\end{theorem}

\begin{wrapfigure}{r}{0.25\textwidth}
  \begin{center}
  \vspace{-7mm}
    \includegraphics[width=0.23\textwidth]{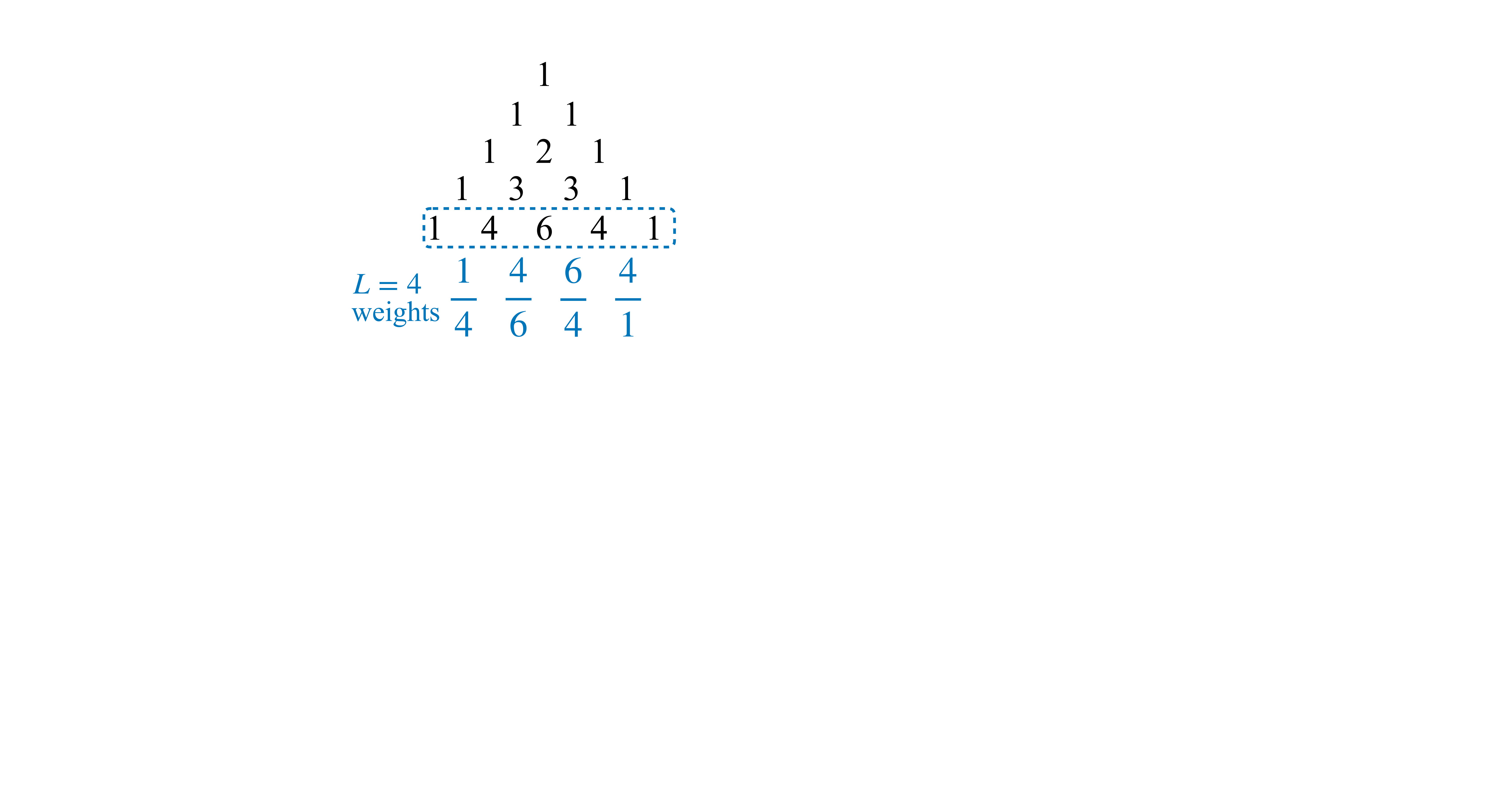}
  \end{center}
  \vspace{-2mm}
  \caption{\textbf{Example of Pascal's triangle} ($\epsilon$ = $1$).} \label{fig_pascal}
    \vspace{-6mm}
\end{wrapfigure}

The result can be extended to other GNNs (Appendix \ref{appendix_lip}). Interestingly, the weights in each level follow a simple rule determined by Pascal's triangle. Figure \ref{fig_pascal} illustrates the weights when $L = 4$, where the weights $w(l)$ gradually decrease from $4 \epsilon$ to $\epsilon/4$ as $l$ becomes smaller. Note that the hyperparameter $w$ of TMD is independent of the parameters of the GNN. Moreover, Theorem \ref{thm_wass_lip} shows that the Lipschitz constant of GNNs under TMD admits a simple form $\prod_{l=1}^{L+1} K_\phi^{(l)}$: the product of Lipschitz constants across layers, similar to fully connected networks \citep{miyato2018spectral, oberman2018lipschitz}.

\subsection{Stability of GNNs under Graph Perturbation}
\label{sec_stability_tmd}

The Lipschitz constant is a common criterion to assess the stability of the neural networks to small perturbations \citep{virmaux2018lipschitz}. Theorem \ref{thm_wass_lip} implies that the output variation of GNNs under graph perturbation can be bounded via the TMD between the original graph and the perturbed one. 
In this section, we analyze the stability of GNNs in more detail by dissecting TMD under three types of graph perturbation: (1) node drop; (2) edge drop; and (3) node feature perturbations. 
\begin{proposition}[Node Drop]
\label{prop_node_drop}
Given a graph $G = (V, E)$, let $G'$ be the graph where node $v \in V$ is dropped. Then the tree mover's distance between $G$ and $G'$ can be bounded by
\begin{align*}
    \TMD_w^L(G, G')\; \leq\; \sum_{l=1}^L \lambda_l \cdot \underbrace{\textnormal{Width}_l(T_v^L)}_{\textnormal{Tree Size}} \cdot \underbrace{\TD_w (T_v^{L-l+1}, T_\mymathbb{0})}_{\textnormal{Tree Norm
    }} ,
\end{align*}
where $\textnormal{Width}_l(T)$ is the width of $l$-th level of tree $T$ and $\lambda_1 = 1$, $\lambda_l = \prod_{j=1}^{l-1} w(L+1-j)$. 
\end{proposition}

The bound is controlled by two factors: (1) tree size and (2) tree norm (distance from the blank tree) of the computation tree $T_v^L$. A node $v$ with a large computation tree size implies that many nodes can be reached from $v$. Deleting $v$ from the graph then significantly changes the computation trees of those reachable nodes, while the magnitude of the variation is controlled by the tree norm. 

\begin{proposition}[Edge Drop]
\label{prop_edge_drop}
Given a graph $G = (V, E)$, let $G'$ be the graph where edge $(u,v) \in E$ is dropped. The tree mover's distance between $G$ and $G'$ can be bounded by
\begin{align*}
    \TMD_w^L(G, G') \leq \sum_{l=1}^{L-1}\lambda_{l+1} \cdot \left( \textnormal{Width}_l(T_v^L) \cdot \TD_w (T_u^{L-l}, T_\mymathbb{0}) + \textnormal{Width}_l(T_u^L) \cdot \TD_w (T_v^{L-l}, T_\mymathbb{0}) \right).
\end{align*}
\end{proposition}
The bound takes a similar form as with the node drop, but includes the effects of both node $v$ and $u$. In particular, the tree size of $v$ ($u$) will control how many computation trees of $u$ ($v$) will be dropped. By using Proposition \ref{prop_node_drop} and \ref{prop_edge_drop}, one can derive bounds for dropping multiple nodes and edges, or even \emph{edge rewiring}. Note that adding nodes or edges is equivalent to the analysis above. 
\begin{figure*}[t]
\begin{center}   
\includegraphics[width=0.99\linewidth]{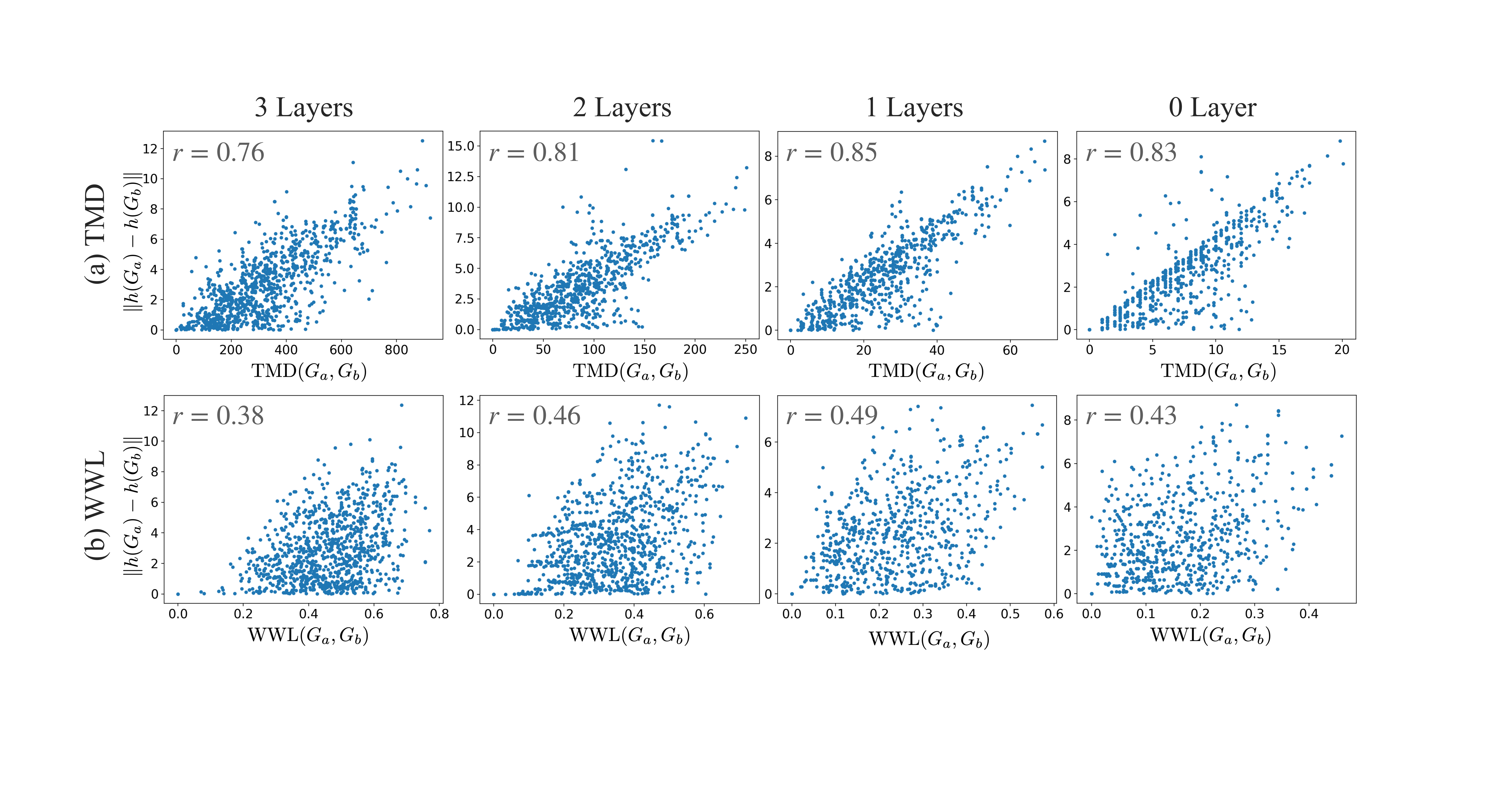}
\end{center}
\vspace{-2mm}
\caption{\textbf{Correlation between GNNs and TMD / WWL.} The Pearson correlation coefficient $r$ between $\| h(G_a) - h(G_b)\|$ and TMD / WWL are showed on the upper left of the figures.  The output variation is highly correlated with TMD, while WWL barely captures the behavior of GNNs with different number of message-passing layers.} 
\vspace{-3mm}
\label{fig_lip}
\end{figure*}

\begin{figure*}[thbp]
\begin{center}   
\includegraphics[width=0.95\linewidth]{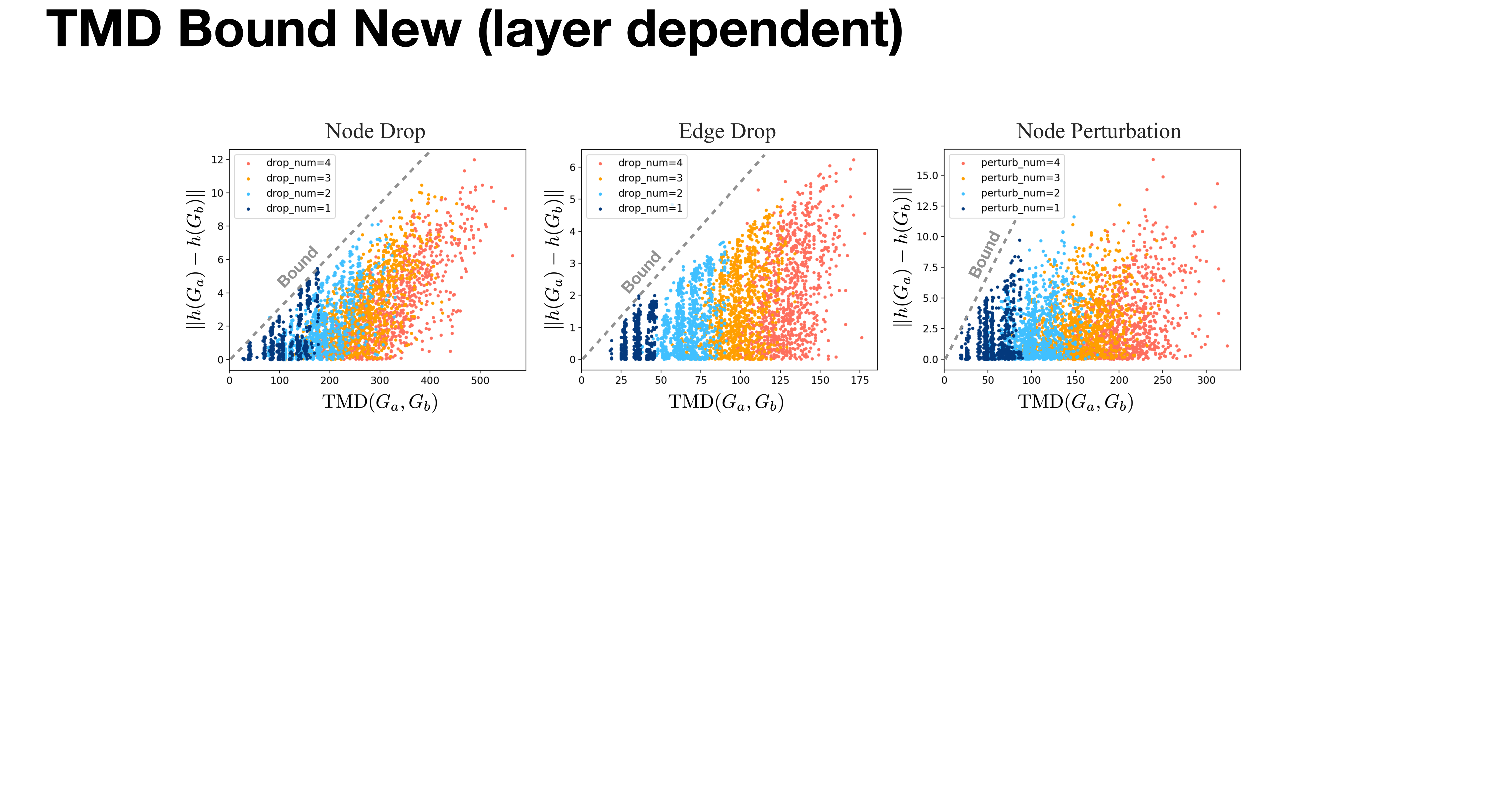}
\vspace{-3mm}
\end{center}
\caption{\textbf{Robustness under Graph Perturbation.} The empirical Lipschitz bounds (dash lines) successfully upper bound the output variation of GNNs under different graph perturbations. 
} 
\vspace{-3mm}
\label{fig_perturb}
\end{figure*}

\begin{proposition}[Node Perturbation]
\label{prop_node_perturb}
Given a graph $G = (V, E)$, let $G'$ be the graph where node feature $x_v$ is perturbed to $x'_v$. The tree mover's distance between $G$ and $G'$ is equal to
\begin{align*}
    \TMD_w^L(G, G') \leq \sum_{l=1}^L \lambda_l \cdot \textnormal{Width}_l(T_v^L) \cdot \left \| x_v - x'_v \right\|.
\end{align*}
\end{proposition}

Different from Proposition \ref{prop_node_drop}, the magnitude of the perturbation is controlled by the norm of the perturbation instead of the norm of the computation tree.

\subsection{Experiments}
\paragraph{Correlation between TMD and GNN output perturbation}
We now empirically examine the theoretical analysis with experiments on the MUTAG dataset \citep{debnath1991structure}. Results for other datasets can be found in Appendix \ref{appendix_exp}. In particular, we measure to what extent candidate graph distances capture input perturbations that lead to output perturbations in the GNN.
We train graph isomorphism networks \citep{xu2018powerful} with varying numbers of message passing layers and plot the relation between input variation $\TMD_w^{L+1}(G_a, G_b)$ and output variation $\| h(G_a) - h(G_b) \|$ for randomly sampled pairs $(G_a, G_b)$ in Figure \ref{fig_lip}. For comparison, we also plot the input variations measured by the recently proposed graph metric WWL \citep{togninalli2019wasserstein}. 
We can see that TMD strongly correlates with the output variation with large Pearson correlation coefficient, supporting the approach of defining the Lipschitz constant with respect to TMD, as in  Theorem \ref{thm_wass_lip}. In contrast, WWL barely captures the input graph perturbations that lead to output variation of GNNs.

\paragraph{Stability under small Perturbations}
Next, we plot the output variation of 3-layer GNNs and the TMD under random graph perturbations in Figure \ref{fig_perturb}. For node perturbations, we change the discrete node attribute for randomly sampled nodes. We additionally plot the Lipschitz bound with empirical estimated Lipschitz constant $\max_{G_a, G_b \in S} \|h(G_a, G_b) \| / \TMD_w^{L+1}(G_a, G_b)$, where $S$ is a set of samples. We refer
reader to \citep{naor2017lipschitz, vacher2021dimension} for analyses on approximation error of estimating Lipschitz constants from finite
samples. We can see that the bound is reasonably tight and estimates the effect of perturbations across different degrees.

\section{Generalization of GNNs under Distribution Shifts}

Finally, we relate the Lipschitz condition of GNNs to the generalization error under distribution shifts
by extending the results from \citep{shen2018wasserstein}. 
Consider a binary classification task with input space $\gX$ and output space $\gY$. In domain adaptation \citep{chuang2020estimating}, a pair of source and target distributions $\mu_S, \mu_T$ over $\gX \times \gY$ are given. Let $p_S$ and $p_T$ denote the respective marginals on the input space $\gX$. In unsupervised domain adaptation, the learning algorithm obtains labelled source samples from $\mu_S$ and unlabelled target samples from $p_T$. To estimate the adaptability of a hypothesis $h$, i.e., its generalization to the target distribution, we aim to bound the target risk $R_T(h) = \E_{x,y \sim \mu_T}[\mathbb{1}_{h(x) \neq y}]$ relative to the source risk $R_S(h) = \E_{x,y \sim \mu_S}[\mathbb{1}_{h(x) \neq y}]$ \citep{ben2010theory,chuang2020estimating,zhao2019learning}. For instance, \citet{shen2018wasserstein} bound the target risk via the source risk and the Wasserstein-1 distance $\gW_1$ between source and target distributions.
\begin{theorem}[\citet{shen2018wasserstein}]\label{thm:domain-adapt}
  For all hypotheses $h \in \mathcal H$, the target risk is bounded as
\begin{align*}
R_{T}(h) \leq  R_{S}(h) + 2 K \gW_1(p_{S}, p_{T}) + \lambda_{\mathcal H},
\end{align*}
where $K$ is the Lipschitz constant of $h$ and $\lambda_{\mathcal H}$ is the best joint risk $\lambda_{\mathcal H} \coloneqq \inf_{h' \in \mathcal{H}}[R_S(h') + R_T(h')]$.
\end{theorem} 
This bound relies on being able to measure Wasserstein distance between the two data distributions, which demands a ground metric on the data space, and an associated Lipschitz constant of the model $h$. 
The TMD and the resulting Lipschitz constant in Section~\ref{sec:lipschitz} make this bound applicable to message passing GNNs, too. In particular, for GNNs that satisfy the Lipschitz constant in Theorem~\ref{thm_wass_lip}, the domain discrepancy $\gW_1(p_{S}, p_{T})$ is defined as
\begin{align}
    \gW_1(p_{S}, p_{T}) = \inf_{\pi \in \Pi(p_S, p_T)} \int \TMD_w^{L+1}(G_a, G_b) d \pi(G_a, G_b),
    \label{eq_wass_div}
\end{align}
where $L$ is the number of message-passing layers in $h$. Since $\gW_1(p_{S}, p_{T})$ can be estimated without labels, Theorem~\ref{thm:domain-adapt} applies to unsupervised domain adaptation. Assuming that there is a model that performs well in both source and target domain, i.e., $\lambda_{\mathcal{H}}$ is small, we may empirically estimate the discrepancy between source and target risk via the GNN Lipschitz constant and $\gW_1(p_{S}, p_{T})$.

\subsection{Experiments}

\begin{figure}[t]
  \begin{minipage}[c]{0.77\textwidth}
    \includegraphics[width=\linewidth]{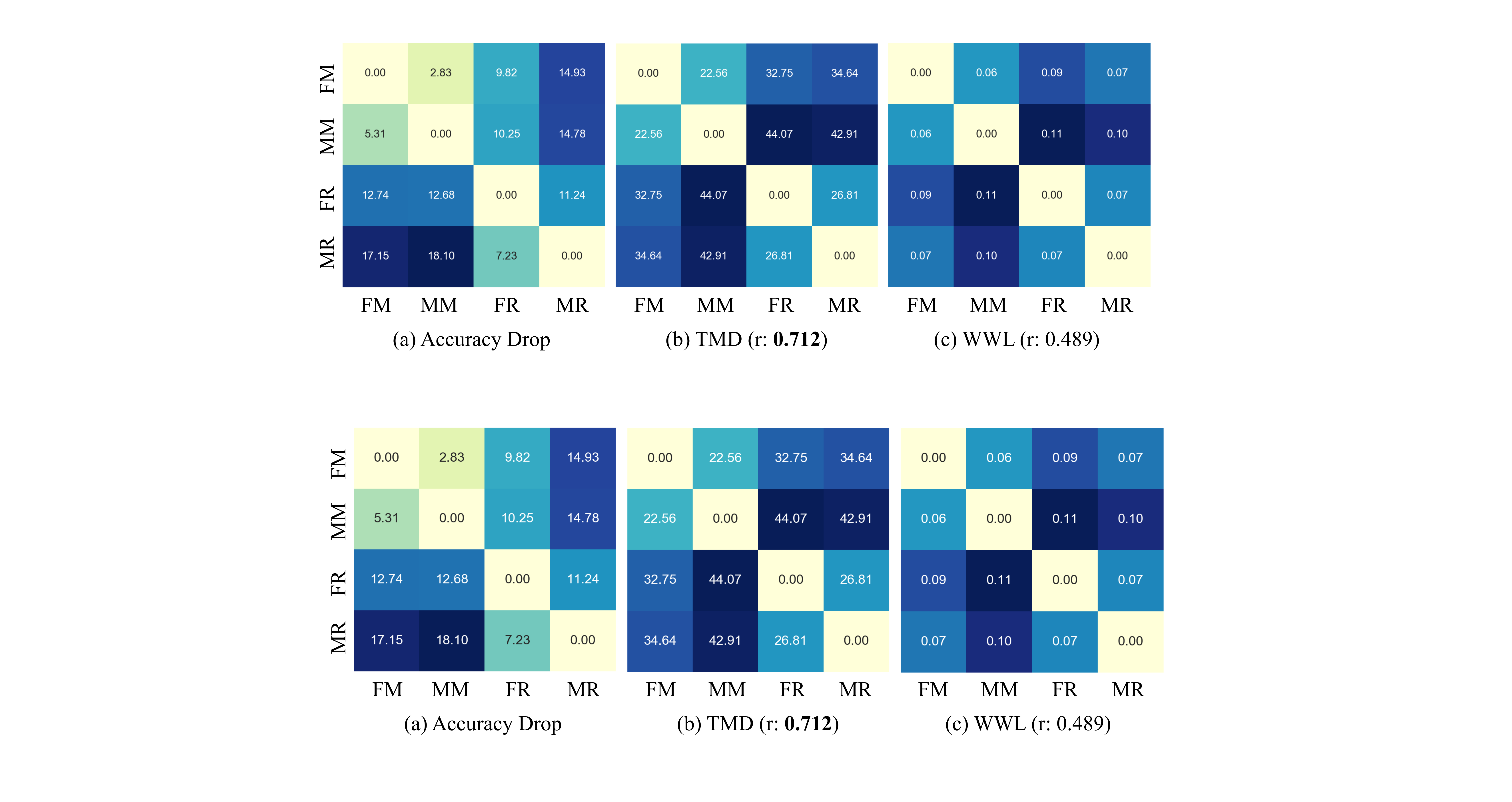}
     \vspace{-6mm}
  \end{minipage}\hfill
  \begin{minipage}[c]{0.22\textwidth}
  \vspace{-6mm}
    \caption{
       \textbf{Accuracy Drop and Distances.} Wasserstein distance based on TMD highly correlates ($r$ = 0.712) with the accuracy drops, while WWL fails to predict the generalization.
    } \label{fig_domain}
    \vspace{-6mm}
  \end{minipage}
  \vspace{-2mm}
\end{figure}

\paragraph{Domain Shifts}
We first verify our analysis on the PTC dataset \citep{helma2001predictive}, which contains carcinogenicity labels of chemical structures for four groups of rodents: male mice (MM), male
rats (MR), female mice (FM) and female rats (FR). We train 3-layer GINs \citep{xu2018powerful} on one group and examine the empirical performance drop on the remaining groups. Figure \ref{fig_domain} shows the Wasserstein distance between groups and the corresponding performance drops. As a baseline, we compute the Wasserstein distance with WWL transportation cost \citep{togninalli2019wasserstein}. The $\gW_1$ distance based on TMD highly correlates with the accuracy drop (Pearson correlation $r=$0.712), while WWL-$\gW_1$ only achieves $r=$0.489.

\paragraph{Size Generalization}
A known challenge for GNNs is generalizing to graphs of different size \citep{yehudai2021local,xu2020neural}. To evaluate TMD on this problem, we sort the PROTEINS dataset \citep{borgwardt2005protein} based on the number of nodes and bin it into 8 subsets, each containing 125 graphs. We again train 3-layer GINs on the smallest and the largest bins and examine the accuracy drops on the remaining ones. Figure \ref{fig_size} plots the accuracy drops with respect to different subsets and the corresponding Wasserstein distance based on TMD and WWL. TMD correlates with the accuracy drop surprisingly well when the models are trained on large graphs and tested on smaller graphs (\textbf{$r=$ 0.97}). The correlation is slightly weaker when the models are trained on small graphs and tested on large graphs ($r=$ 0.83). Yet, TMD shows a much more gradual change in distance, in agreement with the gradual change in drop, than WWL, and hence reflects the overall behavior better.

\begin{figure}[h]
  \begin{minipage}[c]{0.65\textwidth}
    \includegraphics[width=\linewidth]{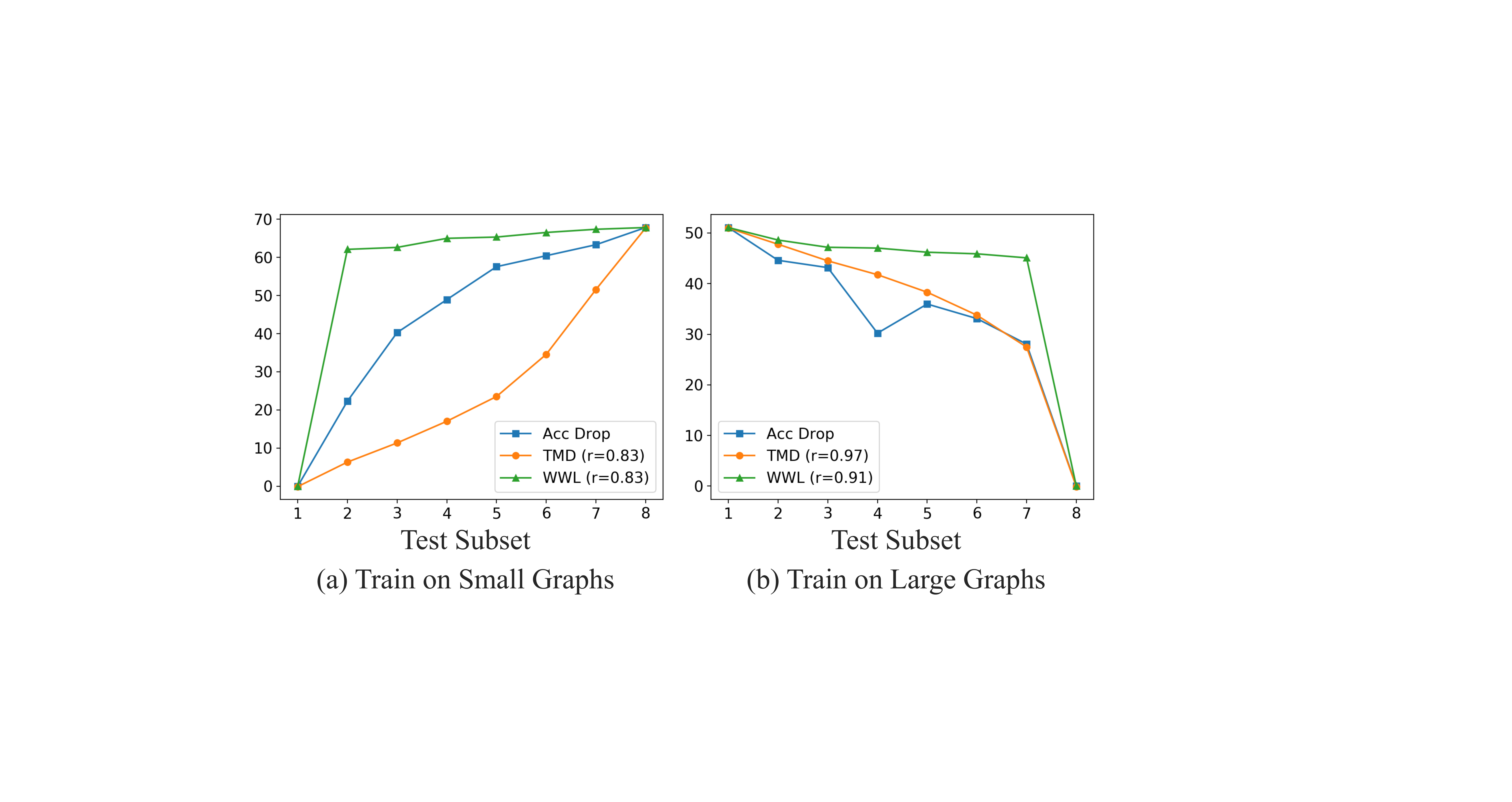}
     \vspace{-6mm}
  \end{minipage}\hfill
  \begin{minipage}[c]{0.33\textwidth}
  \vspace{-2mm}
    \caption{
       \textbf{Size Generalization.} Index 1 denotes the bin of smallest graphs and index 8 the largest graphs. For better visualization, the Wasserstein distances are normalized to make the maximal distance equal to the greatest performance drops, as the absolute scales of distances are less important. The models in (a) / (b) are trained on subset 1 / 8, respectively.
    } \label{fig_size}
    \vspace{-6mm}
  \end{minipage}
  \vspace{-2mm}
\end{figure}

\section{Conclusion}
In this work, we introduce Tree Mover's Distance (TMD), a new graph distance based on optimal transport between computation trees. First, TMD captures the structural and attribute properties that are important for many graph classification tasks. Second, it reflects the patterns that determine the generalization behavior of message passing graph neural networks, and, hence, offers a suitable tool to predict the perturbation stability and out-of-domain generalization capability of such GNNs. Hence, it bears promise in applications, both for graph learning tasks and predicting reliability of graph learning models, and, in theory, as a tool for new tighter analyses of generalization in GNNs. 

\paragraph{Acknowledgements} This work was in part supported by NSF BIGDATA IIS-1741341, NSF AI Institute TILOS, NSF CAREER 1553284. CC is supported
by a IBM PhD Fellowship.

\bibliographystyle{plainnat}
\bibliography{bibfile}

\begin{thebibliography}{60}
\providecommand{\natexlab}[1]{#1}
\providecommand{\url}[1]{\texttt{#1}}
\expandafter\ifx\csname urlstyle\endcsname\relax
  \providecommand{\doi}[1]{doi: #1}\else
  \providecommand{\doi}{doi: \begingroup \urlstyle{rm}\Url}\fi

\bibitem[Altschuler et~al.(2017)Altschuler, Niles-Weed, and
  Rigollet]{altschuler2017near}
Jason Altschuler, Jonathan Niles-Weed, and Philippe Rigollet.
\newblock Near-linear time approximation algorithms for optimal transport via
  sinkhorn iteration.
\newblock \emph{Advances in neural information processing systems}, 30, 2017.

\bibitem[Arvind et~al.(2020)Arvind, Fuhlbr{\"u}ck, K{\"o}bler, and
  Verbitsky]{arvind2020weisfeiler}
Vikraman Arvind, Frank Fuhlbr{\"u}ck, Johannes K{\"o}bler, and Oleg Verbitsky.
\newblock On weisfeiler-leman invariance: Subgraph counts and related graph
  properties.
\newblock In \emph{Journal of Computer and System Sciences}, volume 113, pages
  42--59. Elsevier, 2020.

\bibitem[Baldan et~al.(2017)Baldan, Bonchi, Kerstan, and
  K{\"o}nig]{baldan2017coalgebraic}
Paolo Baldan, Filippo Bonchi, Henning Kerstan, and Barbara K{\"o}nig.
\newblock Coalgebraic behavioral metrics.
\newblock \emph{arXiv preprint arXiv:1712.07511}, 2017.

\bibitem[Barocas et~al.(2017)Barocas, Hardt, and
  Narayanan]{barocas2017fairness}
Solon Barocas, Moritz Hardt, and Arvind Narayanan.
\newblock Fairness in machine learning.
\newblock \emph{Nips tutorial}, 1:\penalty0 2, 2017.

\bibitem[Ben-David et~al.(2010)Ben-David, Blitzer, Crammer, Kulesza, Pereira,
  and Vaughan]{ben2010theory}
Shai Ben-David, John Blitzer, Koby Crammer, Alex Kulesza, Fernando Pereira, and
  Jennifer~Wortman Vaughan.
\newblock A theory of learning from different domains.
\newblock \emph{Machine learning}, 79\penalty0 (1):\penalty0 151--175, 2010.

\bibitem[Bianchi et~al.(2020)Bianchi, Grattarola, and
  Alippi]{bianchi2020spectral}
Filippo~Maria Bianchi, Daniele Grattarola, and Cesare Alippi.
\newblock Spectral clustering with graph neural networks for graph pooling.
\newblock In \emph{International Conference on Machine Learning}, pages
  874--883. PMLR, 2020.

\bibitem[Borgwardt et~al.(2020)Borgwardt, Ghisu, Llinares-L\'{o}pez, O’Bray,
  and Rieck]{borgwardt_book}
Karsten Borgwardt, Elisabetta Ghisu, Felipe Llinares-L\'{o}pez, Leslie
  O’Bray, and Bastian Rieck.
\newblock \emph{Graph Kernels: State-of-the-Art and Future Challenges}.
\newblock Now Foundations and Trends, 2020.

\bibitem[Borgwardt et~al.(2005)Borgwardt, Ong, Sch{\"o}nauer, Vishwanathan,
  Smola, and Kriegel]{borgwardt2005protein}
Karsten~M Borgwardt, Cheng~Soon Ong, Stefan Sch{\"o}nauer, SVN Vishwanathan,
  Alex~J Smola, and Hans-Peter Kriegel.
\newblock Protein function prediction via graph kernels.
\newblock \emph{Bioinformatics}, 21\penalty0 (suppl\_1):\penalty0 i47--i56,
  2005.

\bibitem[Bunke and Shearer(1998)]{bunke1998graph}
Horst Bunke and Kim Shearer.
\newblock A graph distance metric based on the maximal common subgraph.
\newblock \emph{Pattern recognition letters}, 19\penalty0 (3-4):\penalty0
  255--259, 1998.

\bibitem[Chapel et~al.(2020)Chapel, Alaya, and Gasso]{chapel2020partial}
Laetitia Chapel, Mokhtar~Z Alaya, and Gilles Gasso.
\newblock Partial optimal tranport with applications on positive-unlabeled
  learning.
\newblock \emph{Advances in Neural Information Processing Systems},
  33:\penalty0 2903--2913, 2020.

\bibitem[Chizat et~al.(2015)Chizat, Peyr{\'e}, Schmitzer, and
  Vialard]{chizat2015unbalanced}
Lenaic Chizat, Gabriel Peyr{\'e}, Bernhard Schmitzer, and Fran{\c{c}}ois-Xavier
  Vialard.
\newblock Unbalanced optimal transport: geometry and kantorovich formulation.
\newblock 2015.

\bibitem[Chuang and Mroueh(2020)]{chuang2020fair}
Ching-Yao Chuang and Youssef Mroueh.
\newblock Fair mixup: Fairness via interpolation.
\newblock In \emph{International Conference on Learning Representations}, 2020.

\bibitem[Chuang et~al.(2020)Chuang, Torralba, and
  Jegelka]{chuang2020estimating}
Ching-Yao Chuang, Antonio Torralba, and Stefanie Jegelka.
\newblock Estimating generalization under distribution shifts via
  domain-invariant representations.
\newblock \emph{International Conference on Machine Learning}, 2020.

\bibitem[Debnath et~al.(1991)Debnath, Lopez~de Compadre, Debnath, Shusterman,
  and Hansch]{debnath1991structure}
Asim~Kumar Debnath, Rosa~L Lopez~de Compadre, Gargi Debnath, Alan~J Shusterman,
  and Corwin Hansch.
\newblock Structure-activity relationship of mutagenic aromatic and
  heteroaromatic nitro compounds. correlation with molecular orbital energies
  and hydrophobicity.
\newblock \emph{Journal of medicinal chemistry}, 34\penalty0 (2):\penalty0
  786--797, 1991.

\bibitem[Du et~al.(2019)Du, Hou, Salakhutdinov, Poczos, Wang, and
  Xu]{du2019graph}
Simon~S Du, Kangcheng Hou, Russ~R Salakhutdinov, Barnabas Poczos, Ruosong Wang,
  and Keyulu Xu.
\newblock Graph neural tangent kernel: Fusing graph neural networks with graph
  kernels.
\newblock \emph{Advances in neural information processing systems}, 32, 2019.

\bibitem[Dwivedi et~al.(2021)Dwivedi, Luu, Laurent, Bengio, and
  Bresson]{dwivedi2021graph}
Vijay~Prakash Dwivedi, Anh~Tuan Luu, Thomas Laurent, Yoshua Bengio, and Xavier
  Bresson.
\newblock Graph neural networks with learnable structural and positional
  representations.
\newblock \emph{arXiv preprint arXiv:2110.07875}, 2021.

\bibitem[Flamary et~al.(2021)Flamary, Courty, Gramfort, Alaya, Boisbunon,
  Chambon, Chapel, Corenflos, Fatras, Fournier, et~al.]{flamary2021pot}
R{\'e}mi Flamary, Nicolas Courty, Alexandre Gramfort, Mokhtar~Z Alaya,
  Aur{\'e}lie Boisbunon, Stanislas Chambon, Laetitia Chapel, Adrien Corenflos,
  Kilian Fatras, Nemo Fournier, et~al.
\newblock Pot: Python optimal transport.
\newblock \emph{Journal of Machine Learning Research}, 22\penalty0
  (78):\penalty0 1--8, 2021.

\bibitem[Fortin(1996)]{fortin1996graph}
Scott Fortin.
\newblock The graph isomorphism problem.
\newblock 1996.

\bibitem[Gama et~al.(2020)Gama, Bruna, and Ribeiro]{gama2020stability}
Fernando Gama, Joan Bruna, and Alejandro Ribeiro.
\newblock Stability properties of graph neural networks.
\newblock \emph{IEEE Transactions on Signal Processing}, 68:\penalty0
  5680--5695, 2020.

\bibitem[Garg et~al.(2020)Garg, Jegelka, and Jaakkola]{garg2020generalization}
Vikas Garg, Stefanie Jegelka, and Tommi Jaakkola.
\newblock Generalization and representational limits of graph neural networks.
\newblock In \emph{International Conference on Machine Learning}, pages
  3419--3430. PMLR, 2020.

\bibitem[Haussler(1999)]{haussler1999convolution}
David Haussler.
\newblock Convolution kernels on discrete structures.
\newblock Technical report, Technical report, Department of Computer Science,
  University of California~…, 1999.

\bibitem[Helma et~al.(2001)Helma, King, Kramer, and
  Srinivasan]{helma2001predictive}
Christoph Helma, Ross~D. King, Stefan Kramer, and Ashwin Srinivasan.
\newblock The predictive toxicology challenge 2000--2001.
\newblock \emph{Bioinformatics}, 17\penalty0 (1):\penalty0 107--108, 2001.

\bibitem[Jegelka(2022)]{jegelka2022theory}
Stefanie Jegelka.
\newblock Theory of graph neural networks: Representation and learning.
\newblock \emph{arXiv preprint arXiv:2204.07697}, 2022.

\bibitem[Kaufman and Rousseeuw(1990)]{kaufman1990partitioning}
Leonard Kaufman and Peter~J Rousseeuw.
\newblock Partitioning around medoids (program pam).
\newblock \emph{Finding groups in data: an introduction to cluster analysis},
  344:\penalty0 68--125, 1990.

\bibitem[Kenlay et~al.(2020)Kenlay, Thanou, and Dong]{kenlay2020stability}
Henry Kenlay, Dorina Thanou, and Xiaowen Dong.
\newblock On the stability of polynomial spectral graph filters.
\newblock In \emph{ICASSP 2020-2020 IEEE International Conference on Acoustics,
  Speech and Signal Processing (ICASSP)}, pages 5350--5354. IEEE, 2020.

\bibitem[Kenlay et~al.(2021)Kenlay, Thano, and Dong]{kenlay2021stability}
Henry Kenlay, Dorina Thano, and Xiaowen Dong.
\newblock On the stability of graph convolutional neural networks under edge
  rewiring.
\newblock In \emph{ICASSP 2021-2021 IEEE International Conference on Acoustics,
  Speech and Signal Processing (ICASSP)}, pages 8513--8517. IEEE, 2021.

\bibitem[Kipf and Welling(2017)]{kipf2016semi}
Thomas~N Kipf and Max Welling.
\newblock Semi-supervised classification with graph convolutional networks.
\newblock In \emph{International Conference on Learning Representations}, 2017.

\bibitem[Liao et~al.(2020)Liao, Urtasun, and Zemel]{liao2020pac}
Renjie Liao, Raquel Urtasun, and Richard Zemel.
\newblock A pac-bayesian approach to generalization bounds for graph neural
  networks.
\newblock In \emph{International Conference on Learning Representations}, 2020.

\bibitem[Lim et~al.(2022)Lim, Robinson, Zhao, Smidt, Sra, Maron, and
  Jegelka]{lim2022sign}
Derek Lim, Joshua Robinson, Lingxiao Zhao, Tess Smidt, Suvrit Sra, Haggai
  Maron, and Stefanie Jegelka.
\newblock Sign and basis invariant networks for spectral graph representation
  learning.
\newblock \emph{arXiv preprint arXiv:2202.13013}, 2022.

\bibitem[Lloyd(1982)]{lloyd1982least}
Stuart Lloyd.
\newblock Least squares quantization in pcm.
\newblock \emph{IEEE transactions on information theory}, 28\penalty0
  (2):\penalty0 129--137, 1982.

\bibitem[Luss and d'Aspremont(2007)]{luss2007support}
Ronny Luss and Alexandre d'Aspremont.
\newblock Support vector machine classification with indefinite kernels.
\newblock \emph{Advances in neural information processing systems}, 20, 2007.

\bibitem[M{\'e}moli(2011)]{memoli2011gromov}
Facundo M{\'e}moli.
\newblock Gromov--wasserstein distances and the metric approach to object
  matching.
\newblock \emph{Foundations of computational mathematics}, 11\penalty0
  (4):\penalty0 417--487, 2011.

\bibitem[Miyato et~al.(2018)Miyato, Kataoka, Koyama, and
  Yoshida]{miyato2018spectral}
Takeru Miyato, Toshiki Kataoka, Masanori Koyama, and Yuichi Yoshida.
\newblock Spectral normalization for generative adversarial networks.
\newblock In \emph{International Conference on Learning Representations}, 2018.

\bibitem[Morris et~al.(2019)Morris, Ritzert, Fey, Hamilton, Lenssen, Rattan,
  and Grohe]{morris19}
C.~Morris, M.~Ritzert, M.~Fey, W.~L. Hamilton, J.~E. Lenssen, G.~Rattan, and
  M.~Grohe.
\newblock Weisfeiler and leman go neural: Higher-order graph neural networks.
\newblock 2019.

\bibitem[Morris et~al.(2020)Morris, Kriege, Bause, Kersting, Mutzel, and
  Neumann]{morris2020tudataset}
Christopher Morris, Nils~M Kriege, Franka Bause, Kristian Kersting, Petra
  Mutzel, and Marion Neumann.
\newblock Tudataset: A collection of benchmark datasets for learning with
  graphs.
\newblock \emph{arXiv preprint arXiv:2007.08663}, 2020.

\bibitem[Naor and Rabani(2017)]{naor2017lipschitz}
Assaf Naor and Yuval Rabani.
\newblock On lipschitz extension from finite subsets.
\newblock \emph{Israel Journal of Mathematics}, 219\penalty0 (1):\penalty0
  115--161, 2017.

\bibitem[Oberman and Calder(2018)]{oberman2018lipschitz}
Adam~M Oberman and Jeff Calder.
\newblock Lipschitz regularized deep neural networks generalize.
\newblock 2018.

\bibitem[Pearson(1905)]{pearson1905problem}
Karl Pearson.
\newblock The problem of the random walk.
\newblock \emph{Nature}, 72\penalty0 (1865):\penalty0 294--294, 1905.

\bibitem[Ramon and G{\"a}rtner(2003)]{ramon2003expressivity}
Jan Ramon and Thomas G{\"a}rtner.
\newblock Expressivity versus efficiency of graph kernels.
\newblock In \emph{Proceedings of the first international workshop on mining
  graphs, trees and sequences}, pages 65--74, 2003.

\bibitem[Sanfeliu and Fu(1983)]{sanfeliu1983distance}
Alberto Sanfeliu and King-Sun Fu.
\newblock A distance measure between attributed relational graphs for pattern
  recognition.
\newblock \emph{IEEE transactions on systems, man, and cybernetics}, \penalty0
  (3):\penalty0 353--362, 1983.

\bibitem[Sato et~al.(2020)Sato, Yamada, and Kashima]{sato2020fast}
Ryoma Sato, Makoto Yamada, and Hisashi Kashima.
\newblock Fast unbalanced optimal transport on a tree.
\newblock \emph{Advances in neural information processing systems},
  33:\penalty0 19039--19051, 2020.

\bibitem[Scarselli et~al.(2018)Scarselli, Tsoi, and
  Hagenbuchner]{scarselli2018vapnik}
Franco Scarselli, Ah~Chung Tsoi, and Markus Hagenbuchner.
\newblock The vapnik--chervonenkis dimension of graph and recursive neural
  networks.
\newblock \emph{Neural Networks}, 108:\penalty0 248--259, 2018.

\bibitem[Shen et~al.(2018)Shen, Qu, Zhang, and Yu]{shen2018wasserstein}
Jian Shen, Yanru Qu, Weinan Zhang, and Yong Yu.
\newblock Wasserstein distance guided representation learning for domain
  adaptation.
\newblock In \emph{Thirty-second AAAI conference on artificial intelligence},
  2018.

\bibitem[Shervashidze et~al.(2011)Shervashidze, Schweitzer, Van~Leeuwen,
  Mehlhorn, and Borgwardt]{shervashidze2011weisfeiler}
Nino Shervashidze, Pascal Schweitzer, Erik~Jan Van~Leeuwen, Kurt Mehlhorn, and
  Karsten~M Borgwardt.
\newblock Weisfeiler-lehman graph kernels.
\newblock \emph{Journal of Machine Learning Research}, 12\penalty0 (9), 2011.

\bibitem[Singh and P{\'o}czos(2018)]{singh2018minimax}
Shashank Singh and Barnab{\'a}s P{\'o}czos.
\newblock Minimax distribution estimation in wasserstein distance.
\newblock \emph{arXiv preprint arXiv:1802.08855}, 2018.

\bibitem[Togninalli et~al.(2019)Togninalli, Ghisu, Llinares-L{\'o}pez, Rieck,
  and Borgwardt]{togninalli2019wasserstein}
Matteo Togninalli, Elisabetta Ghisu, Felipe Llinares-L{\'o}pez, Bastian Rieck,
  and Karsten Borgwardt.
\newblock Wasserstein weisfeiler-lehman graph kernels.
\newblock \emph{Advances in Neural Information Processing Systems}, 32, 2019.

\bibitem[Vacher et~al.(2021)Vacher, Muzellec, Rudi, Bach, and
  Vialard]{vacher2021dimension}
Adrien Vacher, Boris Muzellec, Alessandro Rudi, Francis Bach, and
  Francois-Xavier Vialard.
\newblock A dimension-free computational upper-bound for smooth optimal
  transport estimation.
\newblock In \emph{Conference on Learning Theory}, pages 4143--4173. PMLR,
  2021.

\bibitem[Van~der Maaten and Hinton(2008)]{van2008visualizing}
Laurens Van~der Maaten and Geoffrey Hinton.
\newblock Visualizing data using t-sne.
\newblock \emph{Journal of machine learning research}, 9\penalty0 (11), 2008.

\bibitem[Vayer et~al.(2019)Vayer, Chapel, Flamary, Tavenard, and
  Courty]{vayer2019optimal}
Titouan Vayer, Laetitia Chapel, R{\'e}mi Flamary, Romain Tavenard, and Nicolas
  Courty.
\newblock Optimal transport for structured data with application on graphs.
\newblock In \emph{ICML 2019-36th International Conference on Machine
  Learning}, pages 1--16, 2019.

\bibitem[Villani(2009)]{villani2009optimal}
C{\'e}dric Villani.
\newblock \emph{Optimal transport: old and new}, volume 338.
\newblock Springer, 2009.

\bibitem[Virmaux and Scaman(2018)]{virmaux2018lipschitz}
Aladin Virmaux and Kevin Scaman.
\newblock Lipschitz regularity of deep neural networks: analysis and efficient
  estimation.
\newblock \emph{Advances in Neural Information Processing Systems}, 31, 2018.

\bibitem[Vishwanathan et~al.(2010)Vishwanathan, Schraudolph, Kondor, and
  Borgwardt]{vishwanathan2010graph}
S~Vichy~N Vishwanathan, Nicol~N Schraudolph, Risi Kondor, and Karsten~M
  Borgwardt.
\newblock Graph kernels.
\newblock \emph{Journal of Machine Learning Research}, 11:\penalty0 1201--1242,
  2010.

\bibitem[Weisfeiler and Leman(1968)]{weisfeiler1968reduction}
Boris Weisfeiler and Andrei Leman.
\newblock The reduction of a graph to canonical form and the algebra which
  appears therein.
\newblock \emph{NTI, Series}, 2\penalty0 (9):\penalty0 12--16, 1968.

\bibitem[Xu et~al.(2021)Xu, Zhang, Li, Du, Kawarabayashi, and
  Jegelka]{xu21extra}
K.~Xu, M.~Zhang, J.~Li, S.~Du, K.~Kawarabayashi, and S.~Jegelka.
\newblock How neural networks extrapolate: From feedforward to graph neural
  networks.
\newblock 2021.

\bibitem[Xu et~al.(2018)Xu, Hu, Leskovec, and Jegelka]{xu2018powerful}
Keyulu Xu, Weihua Hu, Jure Leskovec, and Stefanie Jegelka.
\newblock How powerful are graph neural networks?
\newblock In \emph{International Conference on Learning Representations}, 2018.

\bibitem[Xu et~al.(2020)Xu, Zhang, Li, Du, Kawarabayashi, and
  Jegelka]{xu2020neural}
Keyulu Xu, Mozhi Zhang, Jingling Li, Simon~Shaolei Du, Ken-Ichi Kawarabayashi,
  and Stefanie Jegelka.
\newblock How neural networks extrapolate: From feedforward to graph neural
  networks.
\newblock In \emph{International Conference on Learning Representations}, 2020.

\bibitem[Yehudai et~al.(2020)Yehudai, Fetaya, Meirom, Chechik, and
  Maron]{yehudai2020size}
Gilad Yehudai, Ethan Fetaya, Eli Meirom, Gal Chechik, and Haggai Maron.
\newblock On size generalization in graph neural networks.
\newblock 2020.

\bibitem[Yehudai et~al.(2021)Yehudai, Fetaya, Meirom, Chechik, and
  Maron]{yehudai2021local}
Gilad Yehudai, Ethan Fetaya, Eli Meirom, Gal Chechik, and Haggai Maron.
\newblock From local structures to size generalization in graph neural
  networks.
\newblock In \emph{International Conference on Machine Learning}, pages
  11975--11986. PMLR, 2021.

\bibitem[Zemel et~al.(2013)Zemel, Wu, Swersky, Pitassi, and
  Dwork]{zemel2013learning}
Rich Zemel, Yu~Wu, Kevin Swersky, Toni Pitassi, and Cynthia Dwork.
\newblock Learning fair representations.
\newblock In \emph{International conference on machine learning}, pages
  325--333. PMLR, 2013.

\bibitem[Zhao et~al.(2019)Zhao, Des~Combes, Zhang, and
  Gordon]{zhao2019learning}
Han Zhao, Remi~Tachet Des~Combes, Kun Zhang, and Geoffrey Gordon.
\newblock On learning invariant representations for domain adaptation.
\newblock In \emph{International Conference on Machine Learning}, pages
  7523--7532. PMLR, 2019.

\end{thebibliography}

\newpage
\section*{Checklist}

The checklist follows the references.  Please
read the checklist guidelines carefully for information on how to answer these
questions.  For each question, change the default \answerTODO{} to \answerYes{},
\answerNo{}, or \answerNA{}.  You are strongly encouraged to include a {\bf
justification to your answer}, either by referencing the appropriate section of
your paper or providing a brief inline description.  For example:
\begin{itemize}
  \item Did you include the license to the code and datasets? \answerYes{See Section~\ref{gen_inst}.}
  \item Did you include the license to the code and datasets? \answerNo{The code and the data are proprietary.}
  \item Did you include the license to the code and datasets? \answerNA{}
\end{itemize}
Please do not modify the questions and only use the provided macros for your
answers.  Note that the Checklist section does not count towards the page
limit.  In your paper, please delete this instructions block and only keep the
Checklist section heading above along with the questions/answers below.

\begin{enumerate}

\item For all authors...
\begin{enumerate}
  \item Do the main claims made in the abstract and introduction accurately reflect the paper's contributions and scope?
    \answerYes{}
  \item Did you describe the limitations of your work?
    \answerYes{}
  \item Did you discuss any potential negative societal impacts of your work?
    \answerNA{}
  \item Have you read the ethics review guidelines and ensured that your paper conforms to them?
    \answerYes{}
\end{enumerate}

\item If you are including theoretical results...
\begin{enumerate}
  \item Did you state the full set of assumptions of all theoretical results?
    \answerYes{}
        \item Did you include complete proofs of all theoretical results?
    \answerYes{See the Appendix}
\end{enumerate}

\item If you ran experiments...
\begin{enumerate}
  \item Did you include the code, data, and instructions needed to reproduce the main experimental results (either in the supplemental material or as a URL)?
    \answerYes{In the zip file}
  \item Did you specify all the training details (e.g., data splits, hyperparameters, how they were chosen)?
    \answerYes{See the Appendix}
        \item Did you report error bars (e.g., with respect to the random seed after running experiments multiple times)?
    \answerYes{}
        \item Did you include the total amount of compute and the type of resources used (e.g., type of GPUs, internal cluster, or cloud provider)?
    \answerYes{}
\end{enumerate}

\item If you are using existing assets (e.g., code, data, models) or curating/releasing new assets...
\begin{enumerate}
  \item If your work uses existing assets, did you cite the creators?
    \answerYes{}
  \item Did you mention the license of the assets?
    \answerNA{}
  \item Did you include any new assets either in the supplemental material or as a URL?
    \answerNA{}
  \item Did you discuss whether and how consent was obtained from people whose data you're using/curating?
    \answerNA{}
  \item Did you discuss whether the data you are using/curating contains personally identifiable information or offensive content?
    \answerNA{}
\end{enumerate}

\item If you used crowdsourcing or conducted research with human subjects...
\begin{enumerate}
  \item Did you include the full text of instructions given to participants and screenshots, if applicable?
    \answerNA{}
  \item Did you describe any potential participant risks, with links to Institutional Review Board (IRB) approvals, if applicable?
    \answerNA{}
  \item Did you include the estimated hourly wage paid to participants and the total amount spent on participant compensation?
    \answerNA{}
\end{enumerate}

\end{enumerate}


\newpage
\appendix

\section*{Broader Impact}

Graph Neural Networks are used in many applications with potential societal implications: predictions on social networks, drug design, computational chemistry and materials science, traffic predictions, etc. In many of these applications, the model may be faced with distribution shifts that may lead to a decline in performance. This work is a step towards understanding and estimating such behavior, by understanding and formalizing what kinds of distribution shifts may impact the model.

Indeed, robustness is closely related to fairness, when the distribution shifts are associated with different demographic groups \citep{chuang2020fair, barocas2017fairness, zemel2013learning}. Among the above applications, this may be particularly an issue with social network analysis, and possibly drug design and traffic prediction. In that case, the results in this paper may provide a basis for diagnostic tools be enabling to quantify the amount of potentially risky distribution shifts.

\section{Proof}

\subsection{Preliminaries}
\paragraph{Equivalence between OT and Wasserstein Distance}
Throughout the proof, we will frequently use the equivalence between optimal transport and Wasserstein distance to simplify the notation. Following the setting in section \ref{sec_ot_intro}, let $X = \{x_i\}_{i=1}^m$ and $Y = \{y_i\}_{j=1}^m$ be two multisets, both containing $m$ elements, and let $\gU(X)$ denote the uniform distribution over a multiset $X$. We first note the following equivalence:
\begin{align*}
    \OT_d(X, Y) := \min_{\gamma \in \Gamma(X, Y)} \langle C, \gamma \rangle = m \cdot  \OT_d^\ast(X, Y) = m \cdot \gW_d(\gU(X), \gU(Y))),
\end{align*}
where $\gW_d(P, Q)$ is the Wasserstein distance between $P, Q$ with cost function $d$ defined as follows: 
\begin{align*}
    \gW_d(P, Q) = \inf_{\pi \in \Pi(P, Q)} \int d(x, y) d \pi(x, y).
\end{align*}
The $\Pi$ denotes the set of measure couplings whose marginals are $P$ and
$Q$, respectively. 

\paragraph{Invariance to Additional Augmentation}
Here we show that unnormalized OT is invariant to ``blank'' augmentation, which is an important property for proving the main theorem.
\begin{lemma}
\label{lemma_uot_eq}
Assume we are given two multisets $X = \{x_i\}_{i=1}^m, Y = \{y_j\}_{j=1}^m$ with the same cardinality $m$, where $x_i, y_j \in \gX$ for all $i, j$. Let $d$ be a metric on $\gX$ and $\mymathbb{0} \in \gX$. Then we have
\begin{align*}
    \OT_{d}(X \cup (\mymathbb{0})^n, Y \cup (\mymathbb{0})^n) = \OT_{d}(X, Y).
\end{align*}
\end{lemma}
\begin{proof}
First, by construction, we already have
\begin{align*}
    \OT_{d}(X \cup (\mymathbb{0})^n, Y \cup (\mymathbb{0})^n) \leq \OT_{d}(X, Y),
\end{align*}
as we can always keep the original coupling padded with an identity matrix and get the same cost. Now suppose we adopt a new coupling for $\OT_{d}(X \cup (\mymathbb{0})^n, Y \cup (\mymathbb{0})^n)$ which might lead to smaller cost. Note that any change to the coupling can be decomposed into two steps: (1) permute, change the permutation in the original coupling and (2) decouple, change the coupling $(x - y)$ to $(x - \mymathbb{0}), (y - \mymathbb{0})$ for some pairing $(x, y)$. In the permutation step, the cost is always increasing as we are destroy the original optimal coupling. In the second step, the cost will also increase due to the triangle inequality of cost function $d$:
\begin{align*}
    d(x, y) + d(\mymathbb{0}, \mymathbb{0}) = d(x, y)  \leq d(x, \mymathbb{0}) + d(y, \mymathbb{0}).
\end{align*}
Therefore, we prove that the equality case must hold by contradiction. Note that the proof holds for any augmentation, not only zero vectors.
\end{proof}

\paragraph{Unnormalized OT is a Metric for Multisets}
Finally, we show that if the transportation cost $d$ is a pseudometric, then the unnormalized OT with blank augmentation $\rho$ is also a pseudometric for multisets. 

\begin{lemma}
\label{lemma_ot_metric}
If $d$ is a metric, then $\OT_d(\rho(\cdot, \cdot))$ is a metric over multisets that do not contain $\mymathbb{0}$. If $\mymathbb{0}$ is included, $\OT_d(\cdot, \cdot)$ is a pseudometric.
\end{lemma}
\begin{proof}

The first two axioms, $\OT_d(\rho(X, X)) = 0$ if and only if $X = X$ and $\OT_d(\rho(X, Y)) = \OT_d(\rho(Y, X))$ immediately hold via the property of optimal transport and Wasserstein distance if $\mymathbb{0} \notin X, Y$. The constraint on $\mymathbb{0}$ is due to the fact $\OT_d(\rho(X, X \cup \mymathbb{0})) = 0$, which violates the first axiom. But practically the augmentations are specified to be different from the elements in the sets. 

Next, we prove the Triangle inequality of unnormalized OT. In particular, we will leverage the equivalence of unnormalized OT and Wasserstein distance as follows:
\begin{align*}
    &\OT_d(\rho(X, Y))
    \\
    &= \max(|X|, |Y|) \cdot \gW_d\left( \gU\left( X \bigcup \left(\mymathbb{0}\right)^{\max(|Y| - |X|, 0)} \right), \gU\left( Y \bigcup \left(\mymathbb{0}\right)^{\max(|X| - |Y|, 0)} \right) \right)
    \\
    &= \max(|X|, |Y|, |Z|) \cdot  \gW_{d}\left( \gU\left( X \bigcup \left(\mymathbb{0}\right)^{\max(\max(|Y|, |Z|) - |X|, 0)} \right), \gU\left( Y \bigcup \left(\mymathbb{0}\right)^{\max(\max(|X|, |Z|) - |Y|, 0)} \right) \right) \tag{Lemma \ref{lemma_uot_eq}}
    \\
    &\leq \max(|X|, |Y|, |Z|) \cdot \Big( \gW_{d}\left( \gU\left( X \bigcup \left(\mymathbb{0}\right)^{\max(\max(|Y|, |Z|) - |X|, 0)} \right), \gU\left( Z \bigcup \left(\mymathbb{0}\right)^{\max(\max(|X|, |Y|) - |Z|, 0)} \right) \right)
    \\
    &\quad + \gW_{d}\left( \gU\left( Z \bigcup \left(\mymathbb{0}\right)^{\max(\max(|X|, |Y|) - |Z|, 0)} \right), \gU\left( Y \bigcup \left(\mymathbb{0}\right)^{\max(\max(|X|, |Z|) - |Y|, 0)} \right) \right) \Big) \tag{Triangle Ineq of $\gW$}
    \\
    &= \max(|X|, |Z|) \cdot \gW_{d}\left( \gU\left( X \bigcup \left(\mymathbb{0}\right)^{\max(|Z| - |X|, 0)} \right), \gU\left( Z \bigcup \left(\mymathbb{0}\right)^{\max(|X| - |Z|, 0)} \right) \right)
    \\
    &\quad + \max(|Z|, |Y|) \cdot \gW_{d}\left( \gU\left( Z \bigcup \left(\mymathbb{0}\right)^{\max(|Y| - |Z|, 0)} \right), \gU\left( Y \bigcup \left(\mymathbb{0}\right)^{\max(|Z| - |Y|, 0)} \right) \right) \tag{Lemma \ref{lemma_uot_eq}}
    \\
    &= \OT_d(\rho(X, Z)) + \OT_d(\rho(Z, Y))
\end{align*}
The inequality holds as Wasserstein distance is a pseudometric if the transportation cost is a pseudometric \citep{singh2018minimax, baldan2017coalgebraic}.
\end{proof}

\subsection{Tree Mover's Distance is a Pseudometric}

Via Lemma \ref{lemma_ot_metric}, to show TMD is a pseudometric, we will focus on proving that the transportation cost $\TD$ is a pseudometric between two trees.

\begin{lemma}
\label{lemma_tree_metric}
The distance $\TD_w(\rho(\cdot, \cdot))$ is a metric for two rooted trees if $w(\cdot) > 0$ and node features do not contain zero vectors, where $\rho$ is the blank tree augmentation. Otherwise, it is a pseudometric.
\end{lemma}

\begin{proof}
We first prove that $\TD_w(T_a, T_b) = 0$ if and only if $T_a = T_b$ by induction. In particular, we will focus on the case Depth$(T_a) = $ Depth$(T_b)$, otherwise the statement trivially holds. When depth $=1$, the $\TD_w$ reduces to Euclidean distance, which is a metric. Suppose the statement holds for depth-$k$ trees. Given two tree $T_a, T_b$ with depth equals to $k+1$, their distance is
\begin{align*}
    \TD_w(T_a,T_b)  = \|r_{T_a} - r_{T_b} \| + w(\Dep(T_a)) \cdot \OT_{\TD_w}\left(\rho(\gT_{r_{T_a}}, \gT_{r_{T_b}}) \right).
\end{align*}
Since $\TD_w$ is a metric for depth-k tree by assumption, and $\OT_{\TD_w}(\rho(\cdot, \cdot))$ is a metric if $\TD_w$ is a metric via Lemma \ref{lemma_ot_metric}, $\TD_w(T_a,T_b) = 0$ if and only if $r_{T_a} = r_{T_b}$ and $\gT_{r_{T_a}} = \gT_{r_{T_b}}$, which completes the proof of induction. Nevertheless, similar to Lemma \ref{lemma_ot_metric}, this does not hold if the nodes contains $\mymathbb{0}$, which should not be the case practically. Otherwise, one can simply add a small values to node features to distinguish them from the zero vector.

For the triangle inequality, we give a proof by induction on depth. For the base case depth $= 1$, the tree mover's distance is simply the Wasserstein-1 distance between the distribution of augmented node feature vectors and the transportation cost $\TD_w(\cdot, \cdot)$ reduces to the Euclidean distance, which is a metric and satisfies the Triangle inequality. Next, we show that if $\TD_w(\cdot, \cdot)$ satisfies the Triangle ineuqality for depth-k trees, it is also a metric for depth-(k+1) trees. For trees $T_a$ and $T_b$ with depth $k+1$, introduced the third tree $T_c$ with depth $k+1$, we have
\begin{align*}
    \TD_w(T_a,T_b)  =& \|r_{T_a} - r_{T_b} \| + w(k+1) \OT_{\TD_w}\left(\rho(\gT_{r_{T_a}}, \gT_{r_{T_b}}) \right)
    \\
    \leq& \|r_{T_a} - r_{T_c} \| + \|r_{T_c} - r_{T_b} \| + w(k+1) \OT_{\TD_w} \left(\rho(\gT_{r_{T_a}}, \gT_{r_{T_b}}) \right)
    \\
    \leq&  \|r_{T_a} - r_{T_c} \| + \|r_{T_c} - r_{T_b} \| + w(k+1) \OT_{\TD_w}\left(\rho(\gT_{r_{T_a}}, \gT_{r_{T_c}}) \right) 
    \\
    &\quad\quad\quad + w(k+1) \OT_{\TD_w}\left(\rho(\gT_{r_{T_c}}, \gT_{r_{T_b}}) \right) \tag{Induction hypothesis}
    \\
    =& \TD_w(T_a,T_c) + \TD_w(T_c,T_b).
\end{align*}
The second inequality holds since OT is a pseudometric for multisets over depth-$k$ trees as $\TD_w(\cdot, \cdot)$ satisfies the Triangle ineuqality for depth-k trees via induction hypothesis. Since $\TD_w$ satisfies the triangle inequality for depth-$k+1$ trees, via Lemma \ref{lemma_ot_metric}, TMD also satisfies the triangle inequality, which completes the proof by mathematical induction.
\end{proof}

\subsection{WL Proof}
\begin{theorem} [Discriminative Power of TMD]
\label{thm_tmd_power}
If two graphs $G_a = (V_a, E_a)$, $G_b = (V_b, E_b)$ are determined to be non-isomorphic in WL iteration $L$ and $w(l) > 0 $ for all $0<l \leq L+1$ and $\mymathbb{0} \notin V_a, V_b$, then $\TMD_{w}^{L+1}(G_a,G_b) > 0$.
\end{theorem}
\begin{proof}
We will show that if two nodes have the same subtree, then their WL labels will be the same by induction. The statement holds when depth is 1, as all the WL labels are the same and all the subtrees have the same single node. Suppose the statement holds for depth-$k$ tree. If WL identifies two graphs are non-isomorphic at iteration $k+1$, this means there are at least a pair of nodes whose neighbors have different WL label at $k$-th iteration. This implies that the neightbors have different subtree via induction hypothesis. Therefore, the depth-$k+1$ subtree of the node, which is constructed by appending the subtrees of neightbors to a new node, will also be different. Therefore, the TMD between the subtrees will be greater than zero via Lemma \ref{lemma_tree_metric}, implying that $\TMD_\lambda^{k+1}$ also determine two graphs are non-isomorphic for all $\lambda > 1$. Again, we add the zero vector constraint for the same reason as Lemma \ref{lemma_ot_metric} and \ref{lemma_tree_metric}.
\end{proof}

\subsection{Lipschitz Bound of GNN proof}
\begin{proof}
For simplicity, we set $\epsilon = 1$ through out the proof. We first bound the difference in prediction based on the embedding in the last layer:
\begin{align*}
    \left \| h(G_a) - h(G_b) \right \| &=\left\| \phi^{(L+1)}\left( \sum_{i \in V_a} z_i^{(L)} \right) - \phi^{(L+1)} \left( \sum_{j \in V_b} z_j^{(L)} \right)  \right \|
    \\
    &\leq K_\phi^{(L+1)} \cdot \left\| \sum_{i \in V_a} z_i^{(L)} - \sum_{j \in V_b} z_j^{(L)} \right \|. \tag{Definition of Lipschitz Condition}
\end{align*}

Let $V_a^\rho$ and $V_b^\rho$ be two multisets of nodes after blank tree augmentation: $(V_a^\rho, V_b^\rho) = \rho(V_a, V_b)$, we have
\begin{align}
    \left\| \sum_{i \in V_a} z_i^{(L)} - \sum_{j \in V_b} z_j^{(L)}  \right \| &= \left\| \sum_{i \in V_a^\rho, j \in V_b^\rho} T_{i,j}^{(L, V_a V_b)} \left(z_i^{(L)} -  z_j^{(L)} \right) \right \|
    \\
    &\leq \sum_{i \in V_a^\rho, j \in V_b^\rho} T_{i,j}^{(L, V_a V_b)} \left\|   z_i^{(L)} -  z_j^{(L)}  \right \|
\end{align}
where $T^{(L, V_a^\rho V_b^\rho)}$ is a transportation plan between $V_a^\rho$ and $V_b^\rho$. Note that the inequality holds for any valid transportation plan. Let $\Delta_{T^{(L, V_a^\rho V_b^\rho)}}^{(L)} = \sum_{i \in V_a^\rho, j \in V_b^\rho} T_{i,j}^{(L, V_a^\rho V_b^\rho)} \left\|   z_i^{(L)} -  z_j^{(L)}  \right \|$, we have
\begin{align*}
    &\Delta_{T^{(L, V_a^\rho V_b^\rho)}}^{(L)}
    \\
    \leq& \sum_{i \in V_a^\rho, j \in V_b^\rho} T_{i,j}^{(L, V_a^\rho V_b^\rho)} \left[ \left\| \phi^{(L)} \left( z_i^{(L-1)} +  \sum_{i' \in \gN(i)}[z_i^{(L-1)}] \right) - \phi^{(L)} \left(z_j^{(L-1)} + \sum_{j' \in \gN(i)}[z_j^{(L-1)} ] \right)  \right\| \right]
    \\
    \leq& K_\phi^{(l)} \left( \sum_{i \in V_a^\rho, j \in V_b^\rho} T_{i,j}^{(L, V_a^\rho V_b^\rho)} \left[ \left\|  z_i^{(L-1)} -  z_j^{(L-1)} \right\| + \left\|  \sum_{i' \in \gN(i)}z_i^{(L-1)}   - \sum_{j' \in \gN(i)}z_j^{(L-1)}  \right\|   \right] \right)
    \\
    \leq& K_\phi^{(l)} \Big( \sum_{i \in V_a^\rho, j \in V_b^\rho} T_{i,j}^{(L, V_a^\rho V_b^\rho)} \Bigg[ \left\|  z_i^{(L-1)} -  z_j^{(L-1)} \right\| +   \sum_{i' \in {\gN(i)}^\rho, j' \in {\gN(j)}^\rho} T_{i',j'}^{(L-1, {\gN(i)}^\rho {\gN(j)}^\rho)} \left[ \left\| z_{i'}^{(L-1)} - z_{j'}^{(L-1)} \right\| \right] \Big)
    \\
    =&  K_\phi^{(l)} \Big( \sum_{i \in V_a^\rho, j \in V_b^\rho} T_{i,j}^{(L, V_a^\rho V_b^\rho)} \Bigg[ \left\|  z_i^{(L-1)} -  z_j^{(L-1)} \right\| +  \Delta_{T^{(L-1, {\gN(i)}^\rho {\gN(j)}^\rho)}}^{(L-1)} \Big)
\end{align*}
Here we introduce another transportation plan $T_{i',j'}^{(L-1, {\gN(i)}^\rho {\gN(j)}^\rho)}$ to pair augmented ${\gN(i)}^\rho$ and ${\gN(j)}^\rho$ for all $i \in V_a^\rho, j \in V_b^\rho$ at layer $L-1$. We then further bound first term using similar strategy:
\begin{align*}
    & \sum_{i \in V_a^\rho, j \in V_b^\rho} T_{i,j}^{(L, V_a V_b)} \left[ \left\|  z_i^{(L-1)} -  z_j^{(L-1)} \right\| \right] 
    \\
    =& \sum_{i \in V_a^\rho, j \in V_b^\rho} T_{i,j}^{(L, V_a V_b)} \left[ \left\| \phi^{(L-1)} \left( z_i^{(L-2)} +  \sum_{i' \in \gN(i)}z_{i'}^{(L-2)} \right) - \phi^{(L-1)} \left(z_j^{(L-2)} + \sum_{j' \in \gN(j)}z_{j'}^{(L-2)}  \right)  \right\| \right]
    \\
    \leq&  K_\phi^{(L-1)}  \Bigg( \sum_{i \in V_a^\rho, j \in V_b^\rho} T_{i,j}^{(L, V_a V_b)} \Bigg( \left\|  z_i^{(L-2)} -  z_j^{(L-2)} \right\| + \sum_{i' \in {\gN(i)}^\rho, j' \in {\gN(j)}^\rho} T_{i',j'}^{(L-1, {\gN(i)}^\rho {\gN(j)}^\rho)}  \left\| z_{i'}^{(L-2)} - z_{j'}^{(L-2)} \right\| \Bigg)  \Bigg).
    \\
    =&  K_\phi^{(L-1)}  \Bigg( \sum_{i \in V_a^\rho, j \in V_b^\rho} T_{i,j}^{(L, V_a V_b)} \Bigg( \left\|  z_i^{(L-2)} -  z_j^{(L-2)} \right\| + \Delta_{T^{(L-1, {\gN(i)}^\rho {\gN(j)}^\rho)}}^{(L-2)} \Bigg)  \Bigg).
\end{align*}
Note that we still use the same transportation plan $T_{i',j'}^{(L-1, {\gN(i)}^\rho {\gN(j)}^\rho)} $ to bound the sum difference. Apply this step recursively, the first term will eventually become $\|x_i - x_j \|$ as
\begin{align*}
K_\phi^{(L)} \left\|  z_i^{(L-1)} -  z_j^{(L-1)} \right\| \rightarrow K_\phi^{(L)} K_\phi^{(L-1)} \left\|  z_i^{(L-2)} -  z_j^{(L-2)} \right\| \rightarrow \prod_{m=1}^L K_\phi^{(m)} \left\|  z_i^{(0)} -  z_j^{(0)} \right\| = \prod_{m=1}^L K_\phi^{(m)} \|x_i - x_j \|,
\end{align*}
and the bound will become
\begin{align*}
    \Delta_{\pi_{V_a, V_b}^l}^{(l)} 
    \leq \sum_{i \in V_a^\rho, j \in V_b^\rho} T_{i,j}^{(L, V_a V_b)} \Bigg[ &\left( \prod_{m=1}^L K_\phi^{(m)} \right)  \|x_i - x_j \| +  \sum_{k=1}^L \left(\prod_{m=1}^k K_\phi^{(L+1-m)} \right)  \cdot\Delta_{T^{(L-1, {\gN(i)}^\rho {\gN(j)}^\rho)}}^{(L-k)} \Bigg] 
    \\
    =\E_{(i,j) \sim \pi_{V_a, V_b}^l} \Bigg[ &\left( \prod_{m=1}^L K_\phi^{(m)} \right)  \|x^i - x^j \| 
    \\
    + & \underbrace{ K_\phi^{(L)} \sum_{i' \in {\gN(i)}^\rho, j' \in {\gN(j)}^\rho} T_{i',j'}^{(L-1, {\gN(i)}^\rho {\gN(j)}^\rho)}  \left\| z_{i'}^{(L-1)} - z_{j'}^{(L-1)} \right\|}_{\circled{1}}
    \\
    + &\underbrace{ K_\phi^{(L)} K_\phi^{(L-1)}  \sum_{i' \in {\gN(i)}^\rho, j' \in {\gN(j)}^\rho} T_{i',j'}^{(L-1, {\gN(i)}^\rho {\gN(j)}^\rho)} \|z_{i'}^{(L-2)} - z_{j'}^{(L-2)} \|  }_{\circled{2}}
    \\
    &\cdots 
    \\
    + &\underbrace{ K_\phi^{(L)} K_\phi^{(L-1)} \cdots K_\phi^{(1)}  \sum_{i' \in {\gN(i)}^\rho, j' \in {\gN(j)}^\rho} T_{i',j'}^{(L-1, {\gN(i)}^\rho {\gN(j)}^\rho)} \|x_{i'} - x_{j'} \|  }_{\circled{$L$}} \Bigg].
\end{align*}
Next, we will bound \circled{1}, \circled{2}, $\cdots$ to \circled{$l$} using a similar way as described above. To start, we introduce the second coupling $\pi_{\gN(i'), \gN(j')}^{l-2}$ for all $i', j' \sim \pi_{\gN(i), \gN(j)}^{l-1}$:
\begin{align*}
    \circled{1} 
    \leq & K_\phi^{(L)} \sum_{i' \in {\gN(i)}^\rho, j' \in {\gN(j)}^\rho} T_{i',j'}^{(L-1, {\gN(i)}^\rho {\gN(j)}^\rho)} \Bigg[ K_\phi^{(L-1)} K_\phi^{(L-2)} \cdots K_\phi^{(1)} \|x_{i'} - x_{j'} \|  
    \\
    &+  K_\phi^{(L-1)} \sum_{i'' \in {\gN(i')}^\rho, j'' \in {\gN(j')}^\rho} T_{i'',j''}^{(L-2, {\gN(i')}^\rho {\gN(j')}^\rho)} \left[ \|z_{i''}^{L-2} - z_{j''}^{L-2} \|  \right]
    \\
    &+  K_\phi^{(L-1)} K_\phi^{(L-2)} \sum_{i'' \in {\gN(i')}^\rho, j'' \in {\gN(j')}^\rho} T_{i'',j''}^{(L-2, {\gN(i')}^\rho {\gN(j')}^\rho)} \left[ \|z_{i''}^{L-3} - z_{j''}^{L-3} \| \right]
    \\
    &\cdots
    \\
    &+  K_\phi^{(L-1)} K_\phi^{(L-2)} \cdots K_\phi^{(1)} \sum_{i'' \in {\gN(i')}^\rho, j'' \in {\gN(j')}^\rho} T_{i'',j''}^{(L-2, {\gN(i')}^\rho {\gN(j')}^\rho)} \left[ \|x_{i''} - x_{j''} \| \right] \Bigg]
    \\
    &\hspace{-30pt}:= bound(\circled{1})
\end{align*}

\begin{align*}
    \circled{2} 
    \leq & K_\phi^{(L)} K_\phi^{(L-1)} \sum_{i' \in {\gN(i)}^\rho, j' \in {\gN(j)}^\rho} T_{i',j'}^{(L-1, {\gN(i)}^\rho {\gN(j)}^\rho)} \Bigg[ K_\phi^{(L-2)} K_\phi^{(L-1)} \cdots K_\phi^{(1)} \|x_{i'} - x_{j'} \|  
    \\
    &+  K_\phi^{(L-2)} \sum_{i'' \in {\gN(i')}^\rho, j'' \in {\gN(j')}^\rho} T_{i'',j''}^{(L-2, {\gN(i')}^\rho {\gN(j')}^\rho)} \left[ \|z_{i''}^{L-3} - z_{j''}^{L-3} \| \right]
    \\
    & + K_\phi^{(L-2)} K_\phi^{(L-3)} \sum_{i'' \in {\gN(i')}^\rho, j'' \in {\gN(j')}^\rho} T_{i'',j''}^{(L-2, {\gN(i')}^\rho {\gN(j')}^\rho)} \left[ \|z_{i''}^{L-4} - z_{j''}^{L-4} \| \right]
    \\
    & \cdots
    \\
    &+  K_\phi^{(L-2)} K_\phi^{(L-3)} \cdots K_\phi^{(1)} \sum_{i'' \in {\gN(i')}^\rho, j'' \in {\gN(j')}^\rho} T_{i'',j''}^{(L-2, {\gN(i')}^\rho {\gN(j')}^\rho)} \left[ \|x_{i''} - x_{j''} \| \right] \Bigg]
    \\
    =  & K_\phi^{(L)} \sum_{i'' \in {\gN(i')}^\rho, j'' \in {\gN(j')}^\rho} T_{i'',j''}^{(L-2, {\gN(i')}^\rho {\gN(j')}^\rho)} \Bigg[ K_\phi^{(L-1)} K_\phi^{(L-2)}  \cdots K_\phi^{(1)} \|x_{i'} - x_{j'} \|  
    \\
    &+  K_\phi^{(L-1)} K_\phi^{(L-2)}\sum_{i'' \in {\gN(i')}^\rho, j'' \in {\gN(j')}^\rho} T_{i'',j''}^{(L-2, {\gN(i')}^\rho {\gN(j')}^\rho)} \left[ \|z_{i''}^{L-3} - z_{j''}^{L-3} \| \right]
    \\
    & + K_\phi^{(L-1)}  K_\phi^{(L-2)} K_\phi^{(L-3)} \sum_{i'' \in {\gN(i')}^\rho, j'' \in {\gN(j')}^\rho} T_{i'',j''}^{(L-2, {\gN(i')}^\rho {\gN(j')}^\rho)} \left[ \|z_{i''}^{L-4} - z_{j''}^{L-4} \| \right]
    \\
    & \cdots
    \\
    & + K_\phi^{(L-1)}  K_\phi^{(L-2)} \cdots K_\phi^{(1)} \sum_{i'' \in {\gN(i')}^\rho, j'' \in {\gN(j')}^\rho} T_{i'',j''}^{(L-2, {\gN(i')}^\rho {\gN(j')}^\rho)} \left[ \|x_{i''} - x_{j''} \| \right] \Bigg]
    \\
    &\hspace{-30pt}:= bound(\circled{2})
\end{align*}

We can see that the bounds for \circled{1} and  \circled{2} only differ in the term $ K_\phi^{(l-1)} \sum_{i'' \in {\gN(i')}^\rho, j'' \in {\gN(j')}^\rho} T_{i'',j''}^{(L-2, {\gN(i')}^\rho {\gN(j')}^\rho)} \left[ \|z_{i''}^{L-2} - z_{j''}^{L-2} \right]$. Moreover, one can show that $bound(\circled{k})$ and bound(\circled{k+1}) only differ in $ K_\phi^{(l-1)} \cdots K_\phi^{(l-k)} \sum_{i'' \in {\gN(i')}^\rho, j'' \in {\gN(j')}^\rho} T_{i'',j''}^{(L-2, {\gN(i')}^\rho {\gN(j')}^\rho)} \left[ \|z_{i''}^{L-k-1} - z_{j''}^{L-k-1} \right]$. Therefore, we can merge bound \circled{1} to bound \circled{$l$} and get
\begin{align*}
    \sum_{i=1}^l bound \circled{i} = & K_\phi^{(L)} \sum_{i' \in {\gN(i)}^\rho, j' \in {\gN(j)}^\rho} T_{i',j'}^{(L-1, {\gN(i)}^\rho {\gN(j)}^\rho)} \Bigg[ L \cdot K_\phi^{(L-1)} K_\phi^{(L-2)} \cdots K_\phi^{(1)} \|x_{i'} - x_{j'} \|  
    \\
    &+  K_\phi^{(L-1)} \sum_{i'' \in {\gN(i')}^\rho, j'' \in {\gN(j')}^\rho} T_{i'',j''}^{(L-2, {\gN(i')}^\rho {\gN(j')}^\rho)} \left[ \|z_{i''}^{L-2} - z_{j''}^{L-2} \|  \right]
    \\
    &+2  K_\phi^{(L-1)} K_\phi^{(L-2)} \sum_{i'' \in {\gN(i')}^\rho, j'' \in {\gN(j')}^\rho} T_{i'',j''}^{(L-2, {\gN(i')}^\rho {\gN(j')}^\rho)} \left[ \|z_{i''}^{L-3} - z_{j''}^{L-3} \| \right]
    \\
    &\cdots
    \\
    &+(L-1) K_\phi^{(L-1)} K_\phi^{(L-2)} \cdots K_\phi^{(1)} \sum_{i'' \in {\gN(i')}^\rho, j'' \in {\gN(j')}^\rho} T_{i'',j''}^{(L-2, {\gN(i')}^\rho {\gN(j')}^\rho)} \left[ \|x_{i''} - x_{j''} \| \right] \Bigg]
\end{align*}

Plugging the bound back yields
\begin{align*}
    &= \sum_{i \in V_a^\rho, j \in V_b^\rho} T_{i,j}^{(L, V_a V_b)} \Bigg[ 
    \left( \prod_{m=1}^L K_\phi^{(m)} \right)  \|x^i - x^j \|  
    \\
    &\hspace{100pt}+  K_\phi^{(L)} \sum_{i' \in {\gN(i)}^\rho, j' \in {\gN(j)}^\rho} T_{i',j'}^{(L-1, {\gN(i)}^\rho {\gN(j)}^\rho)} \bigg[L \cdot K_\phi^{(L-1)} K_\phi^{(L-2)} \cdots K_\phi^{(1)} \|x_{i'} - x_{j'} \|  
    \\
    &\hspace{100pt} +\underbrace{K_\phi^{(L-1)} \sum_{i'' \in {\gN(i')}^\rho, j'' \in {\gN(j')}^\rho} T_{i'',j''}^{(L-2, {\gN(i')}^\rho {\gN(j')}^\rho)} \left[ \|z_{i''}^{L-2} - z_{j''}^{L-2} \| \right]}_{\circled{1'}}
    \\
    &\hspace{100pt} + 2 \underbrace{K_\phi^{(L-1)} K_\phi^{(L-2)} \sum_{i'' \in {\gN(i')}^\rho, j'' \in {\gN(j')}^\rho} T_{i'',j''}^{(L-2, {\gN(i')}^\rho {\gN(j')}^\rho)} \left[ \|z_{i''}^{L-3} - z_{j''}^{L-3} \| \right]}_{\circled{2'}}
    \\
    &\hspace{100pt} \cdots
    \\
    &\hspace{100pt} + (L-1) \underbrace{K_\phi^{(L-1)} K_\phi^{(L-2)} \cdots K_\phi^{(1)} \sum_{i'' \in {\gN(i')}^\rho, j'' \in {\gN(j')}^\rho} T_{i'',j''}^{(L-2, {\gN(i')}^\rho {\gN(j')}^\rho)} \left[ \|x_{i''} - x_{j''} \| \right] }_{\small \circled{$(L-1)'$}} \bigg] \Bigg]
\end{align*}
Now we see something familiar: \circled{1'} to \circled{$(L-1)'$} can all be bounded in the same way, e.g., 
\begin{align*}
    \circled{1'} 
    \leq & K_\phi^{(L-1)} \sum_{i'' \in {\gN(i')}^\rho, j'' \in {\gN(j')}^\rho} T_{i'',j''}^{(L-2, {\gN(i')}^\rho {\gN(j')}^\rho)} \Bigg[ K_\phi^{(L-2)} K_\phi^{(L-3)} \cdots K_\phi^{(1)} \|x_{i'} - x_{j'} \|  + 
    \\
    & K_\phi^{(L-2)}\sum_{i''' \in {\gN(i'')}^\rho, j''' \in {\gN(j'')}^\rho} T_{i''',j'''}^{(L-3, {\gN(i'')}^\rho {\gN(j'')}^\rho)} \left[ \|z_{i''}^{L-3} - z_{j''}^{L-3} \|  \right]
    \\
    & K_\phi^{(L-2)} K_\phi^{(L-3)} \sum_{i''' \in {\gN(i'')}^\rho, j''' \in {\gN(j'')}^\rho} T_{i''',j'''}^{(L-3, {\gN(i'')}^\rho {\gN(j'')}^\rho)} \left[ \|z_{i''}^{L-4} - z_{j''}^{L-4} \| \right]
    \\
    &\cdots
    \\
    & K_\phi^{(L-2)} K_\phi^{(L-3)} \cdots K_\phi^{(1)} \sum_{i''' \in {\gN(i'')}^\rho, j''' \in {\gN(j'')}^\rho} T_{i''',j'''}^{(L-3, {\gN(i'')}^\rho {\gN(j'')}^\rho)} \left[ \|x_{i'''} - x_{j'''} \| \right] \Bigg]
    := bound(\circled{1'})
\end{align*}
Similarly, plugging the bound gives
\begin{align*}
    &= \sum_{i \in V_a^\rho, j \in V_b^\rho} T_{i,j}^{(L, V_a V_b)}  \Bigg[ 
    \left( \prod_{m=1}^L K_\phi^{(m)} \right)  \|x^i - x^j \|  
    \\
    &\hspace{20pt}+  K_\phi^{(l)} \sum_{i' \in {\gN(i)}^\rho, j' \in {\gN(j)}^\rho} T_{i',j'}^{(L-1, {\gN(i)}^\rho {\gN(j)}^\rho)} \bigg[ L \cdot K_\phi^{(L-1)} K_\phi^{(L-2)} \cdots K_\phi^{(1)} \|x_{i'} - x_{j'} \|  
    \\
    &\hspace{20pt}+  K_\phi^{(l-2)} \sum_{i'' \in {\gN(i')}^\rho, j'' \in {\gN(j')}^\rho} T_{i'',j''}^{(L-2, {\gN(i')}^\rho {\gN(j')}^\rho)} \bigg[ (1 + 2 + \cdots +(L-1)) \cdot K_\phi^{(L-2)} K_\phi^{(L-3)} \cdots K_\phi^{(1)} \|x_{i'} - x_{j'} \|  
    \\
    &\hspace{40pt}+1 \cdot K_\phi^{(l-2)} \sum_{i''' \in {\gN(i'')}^\rho, j''' \in {\gN(j'')}^\rho} T_{i''',j'''}^{(L-3, {\gN(i'')}^\rho {\gN(j'')}^\rho)} \left[ \|z_{i''}^{L-3} - z_{j''}^{L-3} \|  \right]
    \\
    & \hspace{40pt}+(1 + 2) \cdot  K_\phi^{(l-2)} K_\phi^{(l-3)} \sum_{i''' \in {\gN(i'')}^\rho, j''' \in {\gN(j'')}^\rho} T_{i''',j'''}^{(L-3, {\gN(i'')}^\rho {\gN(j'')}^\rho)} \left[ \|z_{i''}^{L-4} - z_{j''}^{L-4} \| \right]
    \\
    &\hspace{40pt}\cdots
    \\
    & \hspace{40pt}+(1 + 2 + \cdots (l-2)) \cdot K_\phi^{(l-2)} K_\phi^{(l-3)} \cdots K_\phi^{(1)} \sum_{i''' \in {\gN(i'')}^\rho, j''' \in {\gN(j'')}^\rho} T_{i''',j'''}^{(L-3, {\gN(i'')}^\rho {\gN(j'')}^\rho)} \left[ \|x_{i'''} - x_{j'''} \| \right] \Bigg] \Bigg]
\end{align*}

We can see that the weight of the center nodes gradually changes as: $L \rightarrow 1+2+\cdots (L-1) \rightarrow 1 + (1+2) + \cdot (1+2+\cdots+(L-2)) \rightarrow \cdots$. These are exactly the elements at level $L+1$ of the Pascal's triangle. See Figure \ref{fig_pascal} for an illustration. For instance, if $L=4$, the $L+1$ level of Pascal's triangle is $(1, 4, 6, 4, 1) = (1, 1 + (1+2), (1 + 2 + 3), 4, 1)$, which matches the weight of center nodes in each level.

Let $P_{L}^l$ is the $l$-th number at level $L$ of Pascal's triangle. Applying the bound recursively and extracting $\prod_{m=1}^L K_\phi^{(m)}$ gives the following bound:
\begin{align*}
    &\Delta_{T^{(L, V_a^\rho V_b^\rho)}}^{(L)}
    \\
    \leq& \left( \prod_{m=1}^L K_\phi^{(m)} \right) \sum_{i \in V_a^\rho, j \in V_b^\rho} T_{i,j}^{(L, V_a V_b)} \Bigg[ \|x^i - x^j \| + \frac{L}{1} \cdot \sum_{i' \in {\gN(i)}^\rho, j' \in {\gN(j)}^\rho} T_{i',j'}^{(L-1, {\gN(i)}^\rho {\gN(j)}^\rho)} \bigg[ \|x_{i'} - x_{j'} \| 
    \\
    &+ \frac{1 + 2 + \cdots + (L-1)}{L} \cdot \sum_{i'' \in {\gN(i')}^\rho, j'' \in {\gN(j')}^\rho} T_{i'',j''}^{(L-2, {\gN(i')}^\rho {\gN(j')}^\rho)}  \|x_{i'} - x_{j'} \| + \cdots \bigg] \Bigg]
    \\
    &= \left( \prod_{m=1}^L K_\phi^{(m)} \right) \sum_{i \in V_a^\rho, j \in V_b^\rho} T_{i,j}^{(L, V_a V_b)} \Bigg[ \|x^i - x^j \| + \frac{P_{L+1}^L}{P_{L+1}^{L+1}} \cdot \sum_{i' \in {\gN(i)}^\rho, j' \in {\gN(j)}^\rho} T_{i',j'}^{(L-1, {\gN(i)}^\rho {\gN(j)}^\rho)} \bigg[ \|x_{i'} - x_{j'} \| 
    \\
    &+ \frac{P_{L+1}^{L-1}}{P_{L+1}^{L}} \cdot \sum_{i'' \in {\gN(i')}^\rho, j'' \in {\gN(j')}^\rho} T_{i'',j''}^{(L-2, {\gN(i')}^\rho {\gN(j')}^\rho)}  \|x_{i'} - x_{j'} \| + \cdots \bigg] \Bigg]
\end{align*}
By setting each transportation plan to the optimal plan acquired from the TMD OT problem, 
the bound is equivalent the tree mover distance with depth-$L+1$ computation tree, and $w(l) = P_{L+1}^{l-1} / P_{L+1}^{l}$. The one with $\epsilon \neq 1$ can be trivially extended from the current proof.
\end{proof}

\subsection{Tree Mover's Distance Stability Proof}

\subsubsection{Node Drop}
\begin{proof}
We use a illustration to provide the conceptual idea of the proof. 

\begin{figure*}[h]
\begin{center}   
\includegraphics[width=0.99\linewidth]{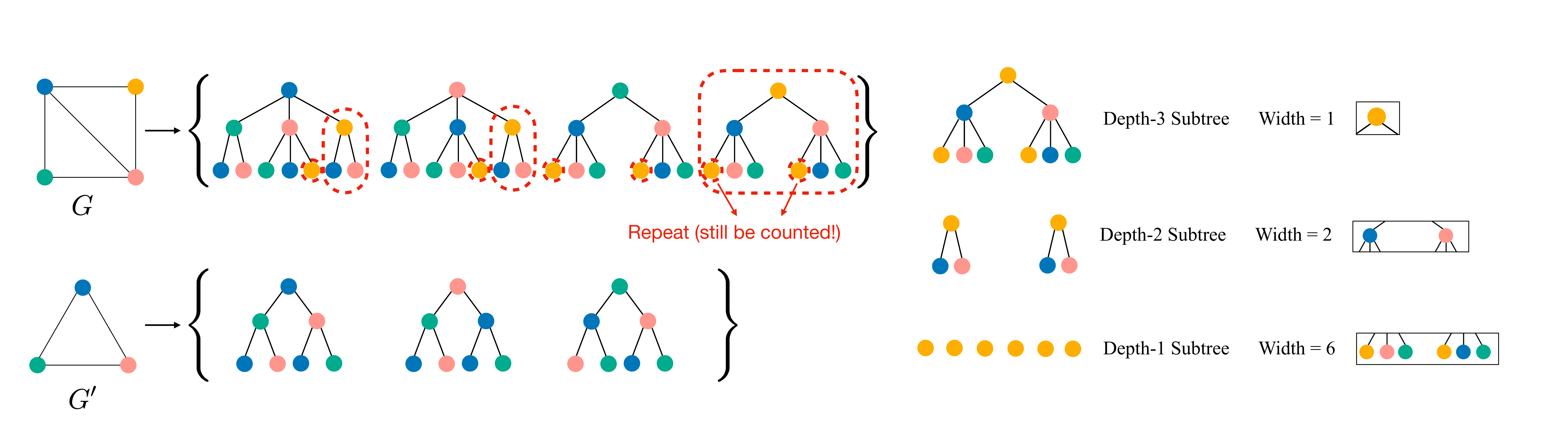}
\end{center}
\vspace{-2mm}
\caption{\textbf{Illustration of Node Drop.}} 
\vspace{-3mm}
\label{fig_nd}
\end{figure*}

Given a graph $G = (V, E)$, let $G'$ be the graph where node $v \in V$ is dropped. The computation tree of node $v$ will play an important role here. Firstly, the nodes in level $l$ of the computation tree are the nodes that can reach $v$ within $l-1$ steps. In particular, the number of nodes in level $l$ determines how many depth-$(L-l+1)$ computation tree of $v$ exists in $\gT_G$. Therefore, deleting $v$ from the graph will also delete all the subtrees in $\gT_G$ that are rooted at $v$, and the number of deleted trees are determined by the width of $T_v$. See Figure \ref{fig_nd} for a illustration. Deleting a subtree $T$ will introduce an additional transportation cost $\TD_w(T, T_\mymathbb{0})$, which is the tree norm of subtree. By aggregating all the tree norms and consider the effect of weights $w$, we arrive at
\begin{align*}
    \TMD_w^L(G, G')\; \leq\; \sum_{l=1}^L \lambda_l \cdot \underbrace{\textnormal{Width}_l(T_v^L)}_{\textnormal{Tree Size}} \cdot \underbrace{\TD_w (T_v^{L-l+1}, T_\mymathbb{0})}_{\textnormal{Tree Norm 
    }} ,
\end{align*}
where $\textnormal{Width}_l(T)$ is the width of $l$-th level of tree $T$ and $\lambda_1 = 1$, $\lambda_l = \prod_{j=1}^{l-1} w(L+1-j)$. This is an upper bound instead of equality as some deleted subtrees are repeated counted in the bound.
\end{proof}

\subsubsection{Edge Drop}
\begin{proof}
We again use a illustration to provide the conceptual idea of the proof. 

\begin{figure*}[h]
\begin{center}   
\includegraphics[width=0.99\linewidth]{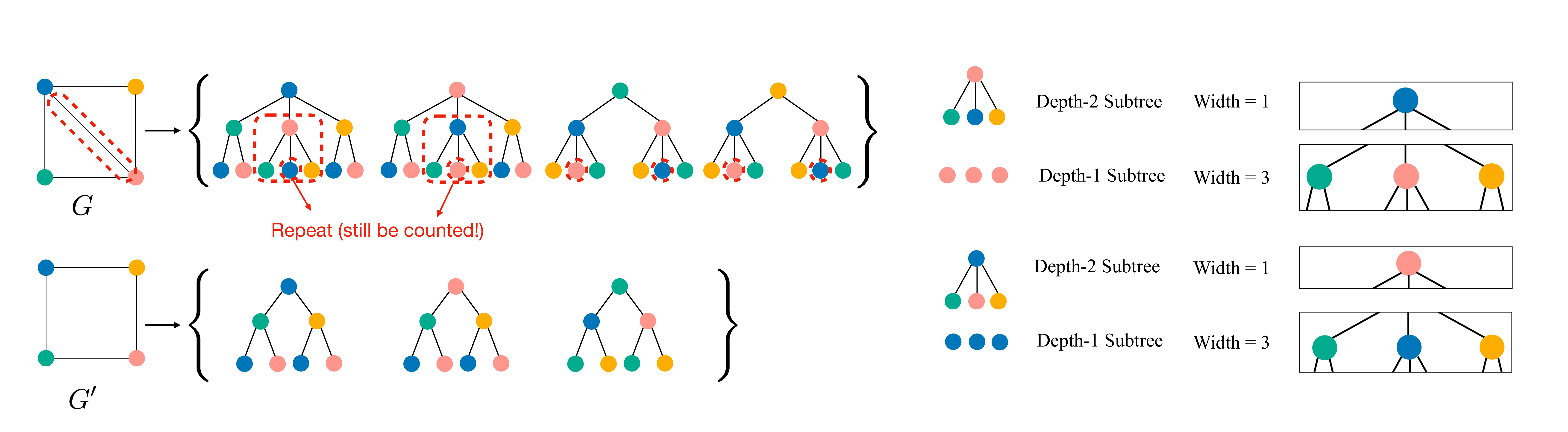}
\end{center}
\vspace{-2mm}
\caption{\textbf{Illustration of Edge Drop.}} 
\vspace{-3mm}
\label{fig_ed}
\end{figure*}

Given a graph $G = (V, E)$, let $G'$ be the graph where edge $(u,v) \in E$ is dropped. Different from node drop, deleting an edge will affect both nodes. In particular, the number of nodes in level $l$ of computation tree $T_v$ determines how many depth-$L-l$ computation tree of $u$ exists in $\gT_G$. Therefore, deleting edge $u-v$ from the graph will also delete all the subtrees in $\gT_G$ that are rooted at $v$ ($u$) where the roots have ancestor $u$ ($v$). See Figure \ref{fig_ed} for a illustration. Similarly, by aggregating all the tree norms and consider the effect of weights $w$, we arrive at
\begin{align*}
    \TMD_w^L(G, G') \leq \sum_{l=1}^{L-1}\lambda_{l+1} \cdot \left( \textnormal{Width}_l(T_v^L) \cdot \TD_w (T_u^{L-l}, T_\mymathbb{0}) + \textnormal{Width}_l(T_u^L) \cdot \TD_w (T_v^{L-l}, T_\mymathbb{0}) \right).
\end{align*}
where $\textnormal{Width}_l(T)$ is the width of $l$-th level of tree $T$ and $\lambda_1 = 1$, $\lambda_l = \prod_{j=1}^{l-1} w(L+1-j)$. This is an upper bound instead of equality as some deleted subtrees are repeated counted in the bound.
\end{proof}

\subsubsection{Node Perturbation}
\begin{proof}
The node perturbation is the simplified case of node drop, where only the node features are perturbed.

\begin{figure*}[h]
\begin{center}   
\includegraphics[width=0.9\linewidth]{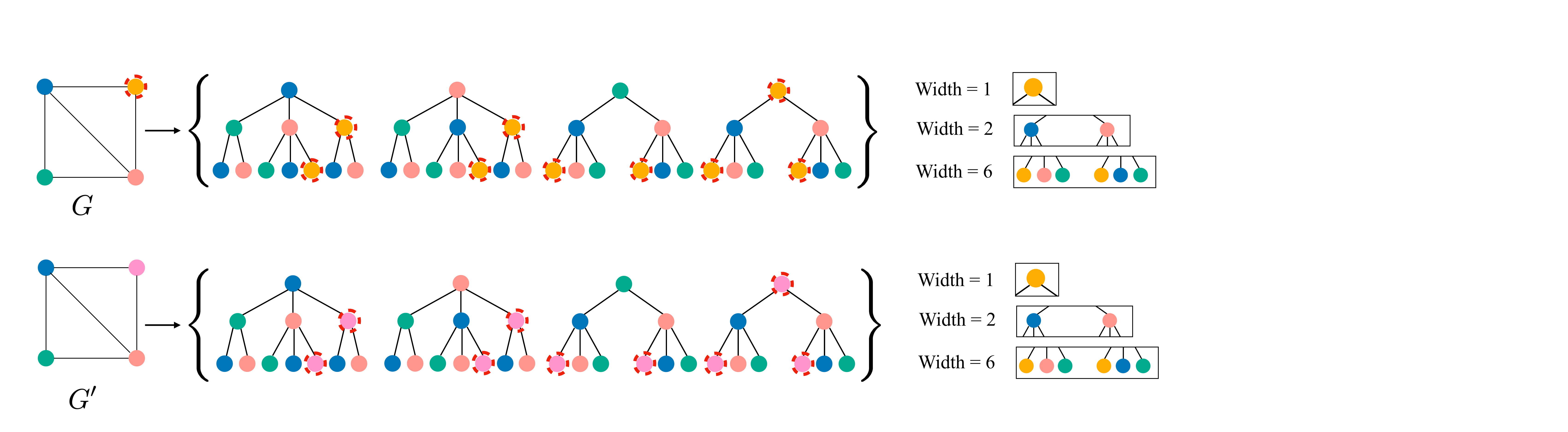}
\end{center}
\vspace{-2mm}
\caption{\textbf{Illustration of Node Perturbation.}} 
\vspace{-3mm}
\label{fig_ed}
\end{figure*}

Given a graph $G = (V, E)$, let $G'$ be the graph where node feature $x_v$ is perturbed to $x'_v$. The tree mover's distance between $G$ and $G'$ is equal to
\vspace{-2mm}
\begin{align*}
    \TMD_w^L(G, G') \leq \sum_{l=1}^L \lambda_l \cdot \textnormal{Width}_l(T_v^L) \cdot \left \| x_v - x'_v \right\|.
\end{align*}
where $\textnormal{Width}_l(T)$ is the width of $l$-th level of tree $T$ and $\lambda_1 = 1$, $\lambda_l = \prod_{j=1}^{l-1} w(L+1-j)$. This is an upper bound instead of equality as changing node features would affect the optimal coupling of OT. We might not have the optimal transportation plan after node perturbation.
\end{proof}

\section{Lipschitz Condition for Other GNNs}
\label{appendix_lip}

\subsection{Graph Convolutional Network}

Here, we consider Graph Convolutional Network (GCN) \citep{kipf2016semi} with the following message passing rules: 
\begin{align*}
    \begin{Large}  \substack{\textnormal{Message} \\ \textnormal{Passing} } \end{Large} \;\; z_v^{(l)} = \phi^{(l)} \left(z_v^{(l-1)} + \epsilon \E_{u \in \gN(v)} z_u^{(l-1)} \right), \;\;\; \begin{Large}  \substack{\textnormal{Graph}\\ \textnormal{Readout}} \end{Large} \;\; h(G) = \phi^{(L+1)} \left( \E_{u \in V} z_u^{(L)} \right).
\end{align*}
In particular, the SUM is replaced with MEAN in the message passing and graph readout. It is easy to show that the same bound holds by replacing all the unnormalized OT with normalized $\OT^\ast$ in tree distance and tree mover's distance, i.e.,
\begin{align*}
 &\TD_{w}^\ast(T_a,T_b) = \begin{cases}
       \|x_{r_{a}} - x_{r_{b}} \| + w(L) \cdot \OT_{\TD_{w}}^\ast(\rho(\gT_{r_a}, \gT_{r_b})) & \text{if $L > 1$}\\
       \|x_{r_{a}} - x_{r_{b}} \| & \text{otherwise},
    \end{cases}
    \\
&\TMD_{w}^{\ast L}(G_a, G_b) = \OT_{\TD_w}^\ast(\rho(\gT_{G_a}^L, \gT_{G_b}^L)).
\end{align*}
This gives the bound:
\begin{align*}
    &\left \| \textnormal{GCN}(G_a) - \textnormal{GCN}(G_b) \right \|  \leq  \prod_{l=1}^{L+1} K_\phi^{(l)}  \cdot  \TMD_{w}^{\ast L+1}(G_a, G_b),
\end{align*}

\subsection{Other Message Passing GNNs}
Sometimes, the center nodes are treated differently compared to the neighbors, e.g., 
\begin{align*}
    \begin{Large}  \substack{\textnormal{Message} \\ \textnormal{Passing} } \end{Large} \;\; z_v^{(l)} = \phi^{(l)} \left(z_v^{(l-1)} + \varphi^{(l)} \left(\sum_{u \in \gN(v)} z_u^{(l-1)} \right) \right), \;\;\; \begin{Large}  \substack{\textnormal{Graph}\\ \textnormal{Readout}} \end{Large} \;\; h(G) = \phi^{(L+1)} \left( \sum_{u \in V} z_u^{(L)} \right)
\end{align*}
The proof can be easily modified as follows:
\begin{align*}
    &\left\| \phi^{(l)} \left( z_i^{(l-1)} +  \varphi^{(l)}\left(\sum_{i' \in \gN(i)}z_i^{(l-1)} \right) \right) - \phi^{(l)} \left(z_j^{(l-1)} + \varphi^{(l)} \left(\sum_{j' \in \gN(i)}z_j^{l-1)} \right) \right)  \right\|
    \\
    \leq& K_\phi^{(l) }\left\|  \left( z_i^{(l-1)} +  \varphi^{(l)}\left(\sum_{i' \in \gN(i)}z_i^{(l-1)} \right) \right) - \left(z_j^{(l-1)} + \varphi^{(l)} \left(\sum_{j' \in \gN(i)}z_j^{(l-1)} \right) \right)  \right\|
    \\
    \leq& K_\phi^{(l) } \left( \left\|   z_i^{(l-1)} - z_j^{(l-1)}  \right \| + \left \| \varphi^{(l)}\left(\sum_{i' \in \gN(i)}z_i^{(l-1)} \right)  -  \varphi^{(l)} \left(\sum_{j' \in \gN(i)}z_j^{(l-1)} \right)  \right\| \right)
    \\
    \leq& K_\phi^{(l) } \left( \left\|   z_i^{(l-1)} - z_j^{(l-1)}  \right \| + K_\varphi^{(l) } \left \| \sum_{i' \in \gN(i)}z_i^{(l-1)}  -   \sum_{j' \in \gN(i)}z_j^{(l-1)}  \right\| \right)
\end{align*}
Therefore, simply replacing the $\epsilon$ with $K_\varphi^{(l)}$ leads to a Lipschitz bound for this variant: 
\begin{align*}
    &\left \| h(G_a) - h(G_b) \right \|  \leq  \prod_{l=1}^{L+1} K_\phi^{(l)}  \cdot  \TMD_{w}^{L+1}(G_a, G_b),
\end{align*}
where $w(l) = K_\varphi^{(l-1)} \cdot P_{L+1}^{l-1} / P_{L+1}^{l}$ for all $l \leq L$ and $P_{L}^l$ is the $l$-th number at level $L$ of Pascal's triangle.

\section{Additional Experiments}
\label{appendix_exp}

\subsection{Graph Clustering}

One advantage of graph metric over GNNs is that we can perform geometric analysis of graph datasets such as graph clustering. We adopt k-medoids \citep{kaufman1990partitioning}, a variant of k-means \citep{lloyd1982least}, to perform unsupervised clustering with TMD. Figure \ref{fig_clustering} provides an qualitative example of clustering, and Table \ref{table_clustering} measures the quallity of clusters with Normalized Mutual Information (NMI) and Completeness Score (CS) \citep{bianchi2020spectral}. We can see that k-medoids with tree mover's distance generates meaningful clusters that aligns with labels.

\begin{figure*}[ht]
\begin{center}   
\includegraphics[width=\linewidth]{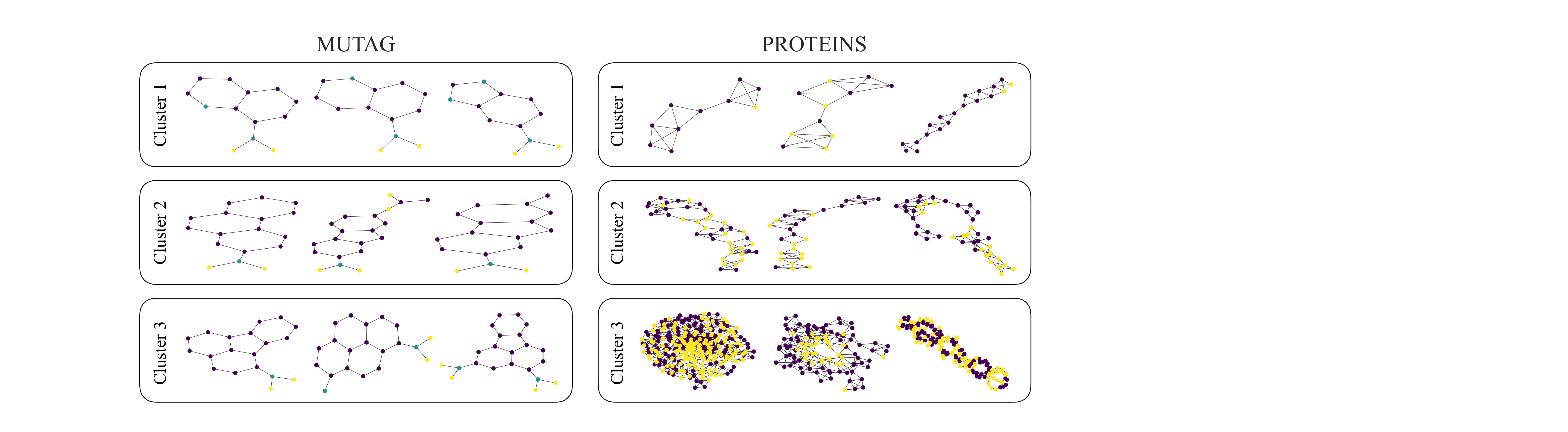}
\end{center}
\vspace{-2mm}
\caption{\textbf{Unsupervised Clustering with TMD.} Node attributes are indicated by colors.} 
\label{fig_clustering}
\end{figure*}

\begin{table*}[!t]
\small
\begin{center}{%
\begin{tabularx}{0.9\textwidth}{l| *{6}{c} }
\toprule
 & \multicolumn{2}{c}{MUTAG (K=2)} & \multicolumn{2}{c}{PROTEINS (K=2)} & \multicolumn{2}{c}{ENZYMES (K=6)}  \\
 & NMI & CS & NMI & CS & NMI & CS 
 \\
\midrule
TMD L=1 & 25.6$\pm$8.1 & 25.0$\pm$7.9 & 6.58$\pm$0.56 & 7.26$\pm$0.61 &  6.55$\pm$0.81 & 6.73$\pm$0.74
\\
TMD L=2 & \textbf{30.4$\pm$8.9} & \textbf{29.7$\pm$8.9} & 7.70$\pm$0.90 & 8.24$\pm$0.69 & \textbf{6.70$\pm$1.11} & \textbf{6.90$\pm$1.01}
\\
TMD L=3 & 28.9$\pm$7.3 & 28.0$\pm$7.2 & 8.28$\pm$0.67 & 8.76$\pm$0.87 & 6.53$\pm$0.65 & 6.69$\pm$0.64
\\
TMD L=4 & 26.6$\pm$5.4 & 25.8$\pm$5.4 & \textbf{9.22$\pm$0.01} & \textbf{9.91$\pm$0.78} & 6.34$\pm$0.60 & 6.57$\pm$0.55
\\
\bottomrule
\end{tabularx}}
\end{center}
\caption{\textbf{Unsupervised Clustering on TU Dataset.} The number of clusters K is equal to
the number of graph classes. The performance is measured by Normalized Mutual
Information (NMI) and Completeness Score (CS).
}
\label{table_clustering}
\end{table*}

\subsection{t-SNE Visualization of Graphs}
Equipped with TMD, we can extend t-SNE \citep{van2008visualizing} from Euclidean space to graphs. Specifically, t-SNE constructs a probability distribution based on \emph{pairwise distance} and minimizes the divergence between the distribution of low dimensional and original data. We simply replace the Euclidean distance in conventional t-SNE with tree mover's distance and show the t-SNE visualization of various graph datasets in Figure \ref{fig_tsne} . Although not perfectly, we can observe the separation between points with different labels.

\begin{figure*}[h]
\begin{center}   
\includegraphics[width=\linewidth]{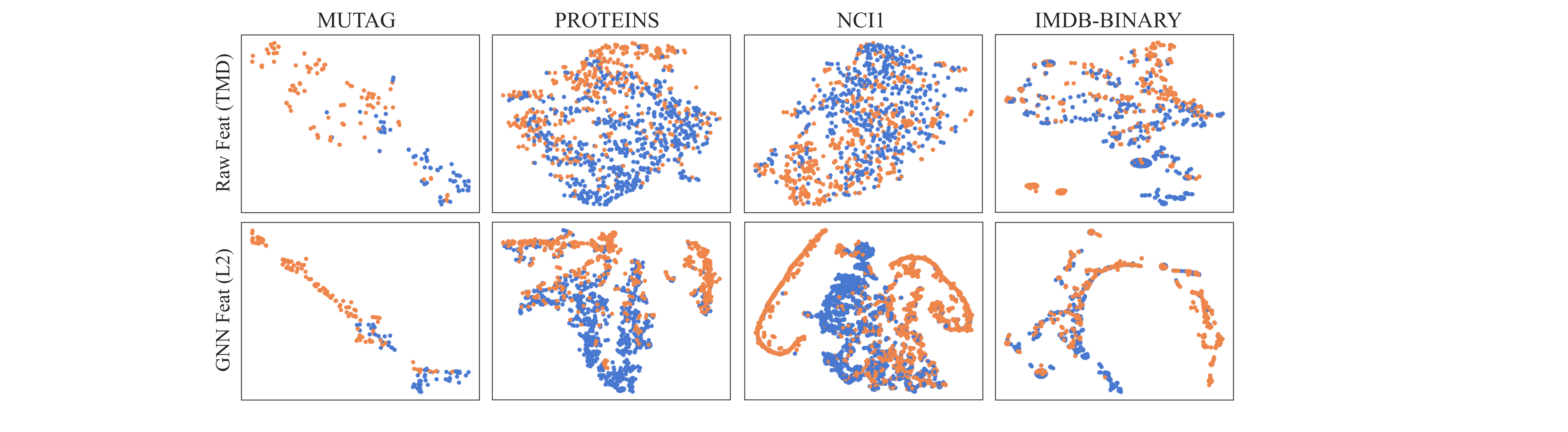}
\end{center}
\caption{\textbf{t-SNE Visualization of Graph Datasets.} The (binary) classes are indicated by colors. The upper row shows the t-SNE visualization of input space with TMD and the bottome rows shows the representation space of trained GNNs with Euclidean distance.} 
\label{fig_tsne}
\end{figure*}

\subsection{Additional Results for Section \ref{sec_stability}}

\begin{figure*}[h]
\begin{center}   
\includegraphics[width=\linewidth]{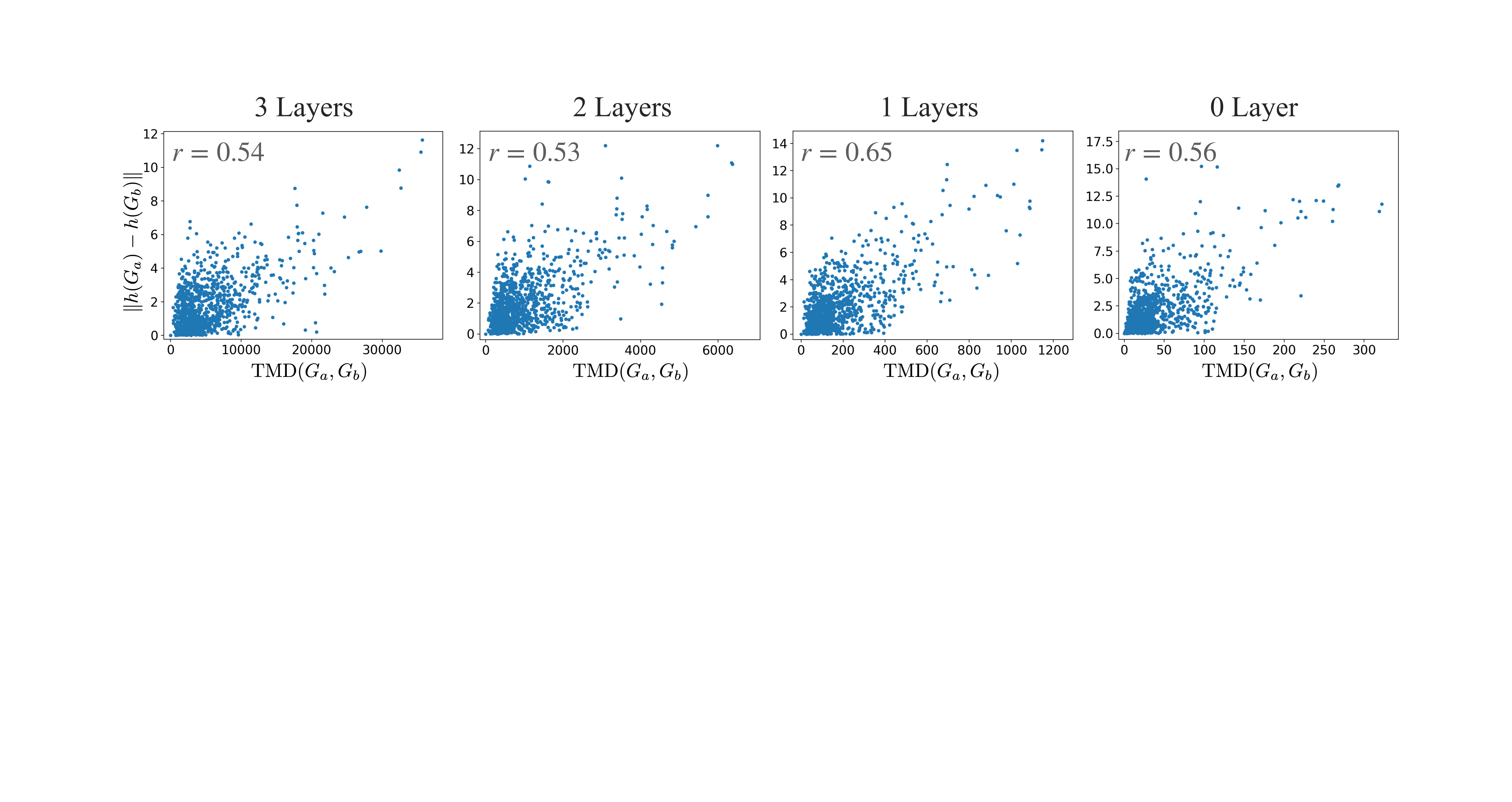}
\end{center}
\caption{\textbf{Correlation between GNNs and TMD on PROTEINS.} The Pearson correlation coefficient $r$ between $\| h(G_a) - h(G_b)\|$ and TMD / WWL are showed on the upper left of the figures.} 
\label{fig_lip_2}
\end{figure*}

We repeat the experiments in section \ref{sec_stability} with PROTEINS dataset and show the results in Figure \ref{fig_lip_2} and \ref{fig_perturb_2}. We can see that the GNNs are less sensitive to small graph perturbation as Figure \ref{fig_perturb_2} shows, as the number of nodes and edges is much larger than the one in MUTAG.

\begin{figure*}[th]
\begin{center}   
\includegraphics[width=\linewidth]{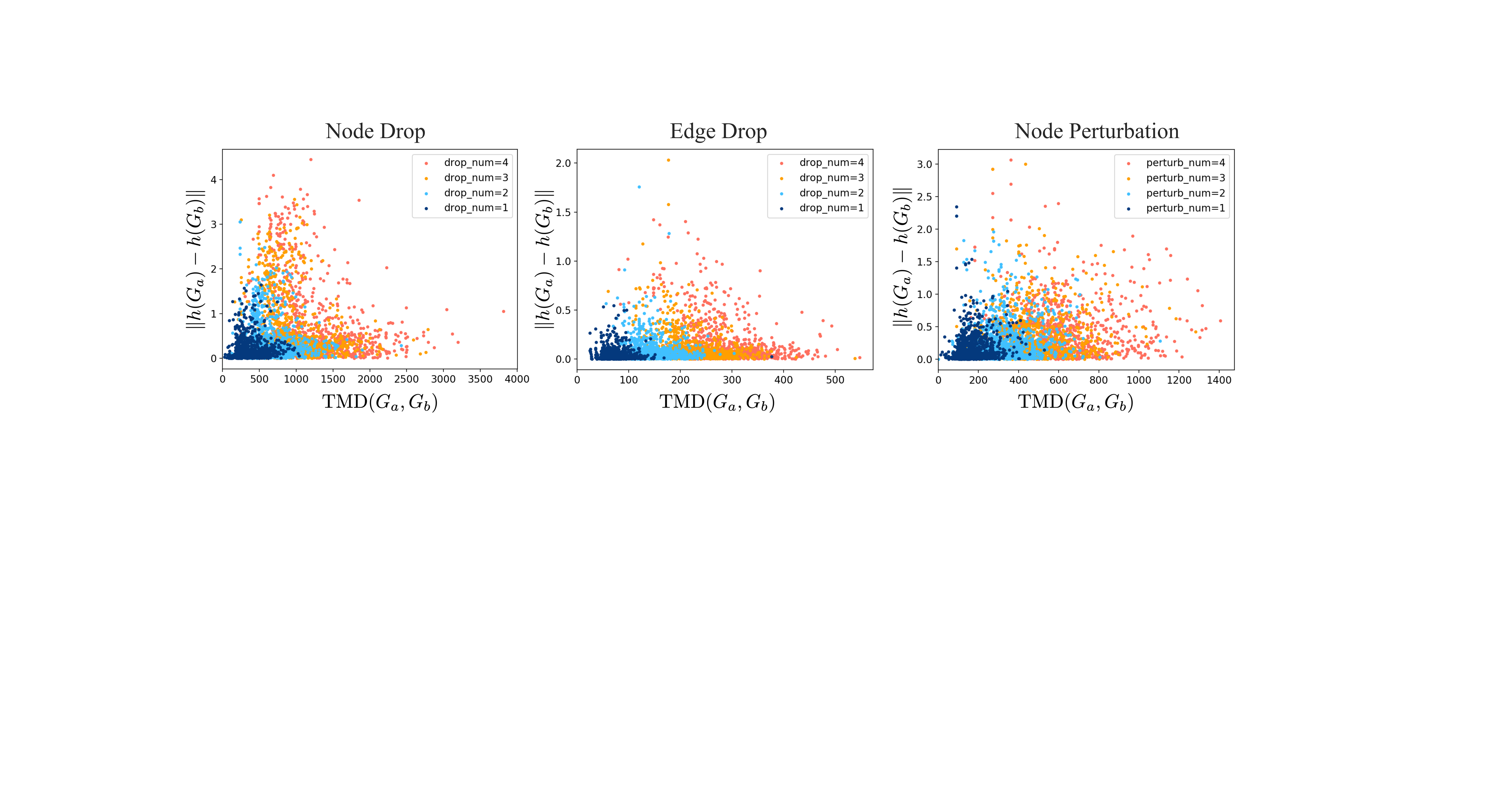}
\end{center}
\caption{\textbf{Robustness under Graph Perturbation on PROTEINS.} } 
\label{fig_perturb_2}
\end{figure*}

\end{document}